\title{Convex Modeling of Interactions with Strong Heredity}
\author{
        Asad Haris\footnote{aharis@uw.edu, Department of Biostatistics}, Daniela Witten\footnote{dwitten@uw.edu, Departments of Statistics and Biostatistics}, and Noah Simon\footnote{nrsimon@uw.edu, Department of Biostatistics} \\
        University of Washington\\
}
\date{\today}

\documentclass[11pt]{article}
\usepackage{graphicx,verbatim,multirow}
\usepackage{amsfonts}
\usepackage{natbib}
\usepackage{float}
\newcommand{\bs}[1]{\boldsymbol{#1}}
\newcommand{\wt}[1]{\widetilde{#1}}
\usepackage{amsmath}
\usepackage[margin=0.8in]{geometry}
\usepackage{setspace}
\usepackage{subfigure}
\usepackage[usenames,dvipsnames,svgnames,table]{xcolor}
\usepackage{amsthm} 
\usepackage{hyperref}
\usepackage{xr}
\hypersetup{
    colorlinks,
    citecolor=black,
    filecolor=black,
    linkcolor=blue,
    urlcolor=black
}

\newcommand{\family}{{\tt FAMILY}}
\newcommand{\iFORM}{{\tt iFORM}}
\newcommand{\diag}{\mathrm{diag}}
\newcommand{\df}{\mathrm{df}}
\newcommand{\familylt}{{\tt FAMILY.l2}}
\newcommand{\familyli}{{\tt FAMILY.linf}}
\newcommand{\familyhiernet}{{\tt FAMILY.hierNet}}

\newcommand{\vanish}{{\tt VANISH}}
\newcommand{\APL}{{\tt APL}}
\newcommand{\hiernet}{{\tt hierNet}}
\newcommand{\glinternet}{{\tt glinternet}}

\newtheorem{theorem}{Theorem}[section]

\newtheorem{lemma}[theorem]{Lemma}
\newtheorem{claim}[theorem]{Claim}

\begin{document}

%\fullpage
\maketitle

\begin{abstract}
We consider the task of fitting a regression model involving interactions among a potentially large set of covariates, in which we wish to enforce strong heredity. We propose \family, a very general   framework for this task. Our proposal  is a generalization of several existing methods, such as \texttt{VANISH} \citep{radchenko2010variable}, \texttt{hierNet} \citep{bien2013lasso}, the all-pairs lasso, and the lasso using only main effects. It can be formulated as the solution to a convex optimization problem, which we solve using an efficient alternating directions method of multipliers (ADMM) algorithm. This algorithm has guaranteed convergence to the global optimum,  can be easily specialized to any convex penalty function of interest, and allows for a straightforward extension to the setting of generalized linear models. We derive an unbiased estimator of the degrees of freedom of \family, and explore its performance in a simulation study and on an HIV sequence data set. 
\end{abstract}

\doublespacing

\section{Introduction}
\label{sec:intro}

\subsection{Modeling Interactions}

\label{sec:overallModel}
In this paper, we model a response variable with a set of main effects and second-order interactions. The problem can be formulated as follows:  we are given a response vector $y$ for $n$ observations, an $n\times p_1$ matrix $X$ of covariates and another $n\times p_2$ matrix $Z$ of covariates. In what follows, the notation $X_{.,j}$ and $Z_{.,k}$ will denote the $j^{th}$ column of $X$ and $k^{th}$ column of Z, respectively. The goal is to fit the model
\begin{equation}
 y_i = B_{0,0} + \sum_{j=1}^{p_1} B_{j,0} X_{i,j} + \sum_{k=1}^{p_2} B_{0,k}Z_{i,k} + \sum_{j=1}^{p_1} \sum_{k=1}^{p_2} B_{j,k}X_{i,j}Z_{i,k} + \varepsilon_i, \ \ i = 1,\ldots , n, 
\label{model}
\end{equation}
where $B$ is a $(p_1+1)\times (p_2+1)$ matrix of coefficients, of which the rows and columns are  indexed from 0 to $p_1$ and 0 to $p_2$ for the variables $X$ and $Z$, respectively.  In the special case where $X=Z$, the coefficient of the $(j,k)^{th}$ interaction is $B_{j,k}+B_{k,j}$, and the coefficient of the $j^{th}$ main effect is  $B_{0,j}+B_{j,0}$.

For brevity, we re-write model (\ref{model}) using array notation. We construct the $n\times(p_1+ 1)\times (p_2+1)$ array $W$ as follows: for $i\in \{1,\ldots,n\}, j\in \{0,\ldots, p_1\}, k\in \{0,\ldots,p_2\}, $

\begin{equation}
W_{i,j,k} = 
\begin{cases}
X_{i,j}Z_{i,k} & \mbox{ for }  j\not= 0 \mbox{ and }  k\not=0\\
X_{i,j} & \mbox{ for }  k=0 \mbox{ and } j\not= 0\\
Z_{i,k} &\mbox{ for }   j=0 \mbox{ and } k \not= 0\\
1 & \mbox{ for }   j=k=0
\end{cases}.
\label{wmat}
\end{equation}
Then (\ref{model}) is equivalent to the model 
\begin{equation}
y = W * B + \varepsilon, 
\label{arraynot}
\end{equation}
where $B$ is the matrix of coefficients as in (\ref{model}), and $W*B$ denotes the $n$-vector whose $i^{th}$ element  takes the form
 $(W*B)_i \equiv \sum_{j=0}^{p_1}\sum_{k=0}^{p_2} W_{i,j,k}B_{j,k}$. The model is displayed in the left panel of Figure $\ref{fig1}$.

 In fitting models with interactions, we may wish to impose either \emph{strong} or \emph{weak} heredity \citep{hamada1992analysis, yates1978design, chipman1996bayesian, joseph2006bayesian},  defined as follows:
\begin{list}{}{}
\item{\emph{Strong Heredity:}} If an interaction term is included in the model, then  both of the corresponding main effects must be present. That is, if $ B_{j,k} \not= 0$, then $B_{j,0} \not= 0$ \emph{and} $B_{0,k} \not= 0$.
\item{\emph{Weak Heredity:}} If an interaction term is included in the model, then at least one of the corresponding main effects must be present. That is, if $ B_{j,k} \not= 0$, then either $B_{j,0} \not= 0$ \emph{or} $B_{0,k} \not= 0$.
\end{list}
  Such constraints facilitate model interpretation \citep{mccullagh1984generalized}, improve statistical power  \citep{cox1984interaction}, and simplify experimental designs \citep{bien2013lasso}.
In this paper we propose a general convex regularized regression approach which naturally and efficiently enforces strong heredity.

\subsection{Summary of Previous Work}
\label{sec:prev.work}

A number of authors have considered the task of fitting interaction models under strong or weak heredity constraints. Constraints to enforce heredity \citep{peixoto1987hierarchical,friedman1991multivariate,bickel2010hierarchical,park2008penalized,wu2010screen} have been applied to conventional step-wise model selection techniques \citep[][chap. 10]{montgomery2012introduction}. \citet{chipman1996bayesian} and \citet{george1993variable} proposed Bayesian methods.  In more recent work, \cite{hao2014interaction} proposed  \iFORM, an approach that performs forward selection on the main effects, and allows interactions into the model once the main effects have already been selected. \iFORM \; has a number of attractive properties, including suitability for the ultra-high-dimensional setting, computational efficiency, as well as proven theoretical guarantees.

In this paper, we take a regularization approach to inducing strong heredity. A number of regularization  approaches for this task have already been proposed in the literature; 
in fact, a strength of our proposal is that it provides a unified framework (and associated algorithm) of which several existing approaches can be seen as special cases.  
 \citet{choi2010variable} propose a non-convex approach, which amounts to a lasso  \citep{tibshirani1996regression} problem with re-parametrized coefficients. 
  Alternatively, some authors have enforced strong or weak heredity via convex penalties or constraints. \citet{jenatton2011structured} and \citet{zhao2009composite} describe a set of   penalties that can be applied to a broad class of problems. As a special case they consider interaction models with strong or weak heredity; this has been further developed by \citet{bach2012structured}. \citet{radchenko2010variable}, \citet{lim2013learning} and \citet{bien2013lasso} propose penalties specifically designed for interaction models with sparsity and strong heredity.  
We now describe the latter two approaches in greater detail.

\subsubsection{\texttt{hierNet} \citep{bien2013lasso}}
\label{sec:hier}
The \texttt{hierNet} approach of \citet{bien2013lasso} fits the model (\ref{model}) with $X=Z$ and $p_1=p_2=p$. In the case of strong heredity, using the notation of \eqref{arraynot}, they  consider the problem
%\begin{equation}
%\begin{split}
%\underset{B\in \mathbb{R}^{(p+1)\times (p+1)}}{\text{minimize}}  \   &\frac{1}{2} \|y-W*B\|_2^2+ \lambda \sum_{j=1}^{p} |B_{j,0}| +  \frac{\lambda}{2} \|B_{-0,-0}\|_1\\
%\text{subject to } & B = B^T, 
%\|B_{j,-0}\|_1 \le |B_{j,0}|
% \text{ for } j = 1,\ldots,p. 
%\end{split}
%\label{hier}
%\end{equation} 
% The constraint  $\|B_{j,-0}\|_1 \le |B_{j,0}|$ in (\ref{hier}) imposes strong heredity. Since (\ref{hier}) is non-convex, \citet{bien2013lasso} instead solve the convex relaxation
\begin{equation}
\begin{split}
 \underset{ B \in \mathbb{R}^{(p+1)\times (p+1)} ,\  
\beta^\pm \in \mathbb{R}^p }{\text{minimize}}  \   & \frac{1}{2}\|y-W*B\|_2^2+ \lambda \sum_{j=1}^{p}(\beta_j^+ + \beta_j^-) +  \frac{\lambda}{2} \|B_{-0,-0} \|_1\\
\text{subject to } & B = B^T, \ B_{0,-0} =  \beta^+ -\beta^- \\ 
&\|B_{j,-0}\|_1 \le \beta_j^++\beta_j^-, \  
 \beta_j^+\ge 0 ,\beta_j^- \ge 0 
 \text{ for } j = 1,\ldots,p.
\end{split}
\label{hier2}
\end{equation}
Using this notation, the coefficient  for the $j^{th}$ main effect is  ${B}_{0,j}+{B}_{j,0}$, and the coefficient  for the $(j,k)^{th}$ interaction is  ${B}_{j,k}+{B}_{k,j}$. Strong heredity is imposed by the constraint $\|B_{j,-0}\|_1 \le \beta_j^++\beta_j^-$.

\subsubsection{\texttt{glinternet} \citep{lim2013learning} }
\label{sec:glin}
Like  \texttt{hierNet}, the \texttt{glinternet} proposal of \citet{lim2013learning}  fits (\ref{model}) with $X=Z$ and $p_1=p_2=p$. In order to describe this approach, we introduce some additional notation. Let $\alpha_k$ be the coefficient of the $k^{th}$ main effect. We decompose $\alpha_k$ into $p$ parameters, i.e. $\alpha_k = \alpha_k^{(0)}+\alpha_k^{(1)}+\ldots+ \alpha_k^{(k-1)}+\alpha_k^{(k+1)} + \ldots + \alpha_k^{(p)}$. We let $\alpha_{jk}+\alpha_{kj}$ denote the coefficient for the interaction between $X_j$ and $X_k$. 
 \citet{lim2013learning} propose to solve the optimization problem   
\begin{equation}
\begin{split}
\underset{ \parbox{1.6in}{ $\alpha_0, \{\alpha_{ij} \}_{i\not=j; i,j\not =0},\\  \{  \alpha_i^{(j)} \}_{ j\not=i}  \in  \mathbb{R}$ }  }{\text{minimize}} &\left\|y- \alpha_0 - \sum_{k=1}^{p} \sum_{j\not= k} \alpha_k^{(j)}X_{.,k} - \sum_{j\not= k} \alpha_{jk} \left( X_{.,j} * X_{.,k} \right) \right\|_2^2\\
 &+ \lambda \left( \sum_{j=1}^{p} |\alpha^{(0)}_{j}| +\sum_{j \neq k}\sqrt{\left(\alpha_j^{(k)} \right)^2 + \left(\alpha_k^{(j)} \right)^2 + \alpha^2_{jk} }  \right),
\end{split}
\label{glin}
\end{equation}
 where $X_{.,j} * X_{.,k}$ denotes element-wise multiplication. Strong heredity is enforced via the group lasso \citep{yuan2006model} penalties: if either $\alpha_{jk}$ or $\alpha_{kj}$ is estimated as non-zero, then $\alpha_j^{(k)}$ and $\alpha_{k}^{(j)}$ will be estimated to be non-zero, and hence so will $\alpha_j$ and $\alpha_k$.

\subsection{Organization of Paper}

The rest of this paper is organized as follows.  In Section \ref{sec:proposal},  we provide details of \family, our proposed approach for modeling interactions. 
An unbiased estimator for its degrees of freedom is in Section~\ref{sec:dof}, and an extension to weak heredity is in Section~\ref{sec:weak}. 
We  explore \family's empirical performance in  simulation in Section \ref{sec:SimStudy}, and in an application to an HIV data set  in Section \ref{sec:realData}. The Discussion is  in Section \ref{sec:conclusion}.

\section{Modeling Interactions with  \family}
\label{sec:proposal}

We propose a \emph{\underline{f}r\underline{a}mework for \underline{m}odeling \underline{i}nteractions with a convex pena\underline{l}t\underline{y}}  (\family). The \family \,  approach is the solution to a convex optimization problem, which (using the notation of Section~\ref{sec:overallModel}) takes the form 
\begin{equation}
\begin{split}
\underset{B\in \mathbb{R}^{(p_1+1)\times (p_2+1)}}{\text{minimize}} \frac{1}{2n} \|y-W*B\|_2^2 +  \lambda_1 \sum_{j=1}^{p_1}P_r( B_{j,.}) + \lambda_2 \sum_{k=1}^{p_2}P_c (B_{.,k}) + \lambda_3 \|B_{-0,-0}\|_1.
\end{split}
\label{obj}
\end{equation}
Here, $\lambda_1$, $\lambda_2$, and $\lambda_3$ are non-negative tuning parameters.  $P_r$ and $P_c$ are convex penalty functions on the rows and columns of the coefficient matrix $B$. 
 The $\| B_{-0,-0} \|_1$ term  denotes the element-wise $\ell_1$-norm on the interactions, which enforces sparsity on the interaction coefficients when $\lambda_3$ is large. 
 The right panel of Figure \ref{fig1} demonstrates the action of each penalty on the matrix $B$. 

As we will see, the choice of $P_r$ and $P_c$ will determine the type of structure (such as strong heredity) enforced on the fitted model. In the examples that follow, we take $P_r=P_c$; however, in principle, these two penalty functions need not be equal. For instance, if the features in $Z$ are known to be of scientific importance, we might choose to perform feature selection on the main effects of $X$ only. In this case, we might choose
 to use $P_r(b) = \|b \|_2$ and $P_c(b) = 0$.

We suggest standardizing the columns of $X$ and $Z$ to have mean zero and variance one before solving  \eqref{obj}, in order to ensure that the main effects and interactions are on the same scale, as is standard practice for penalized regression estimators \citep{hastie2009elements}. We take this approach in Sections~\ref{sec:SimStudy} and \ref{sec:realData}.

\begin{figure}
\centering
\includegraphics[scale = 0.45]{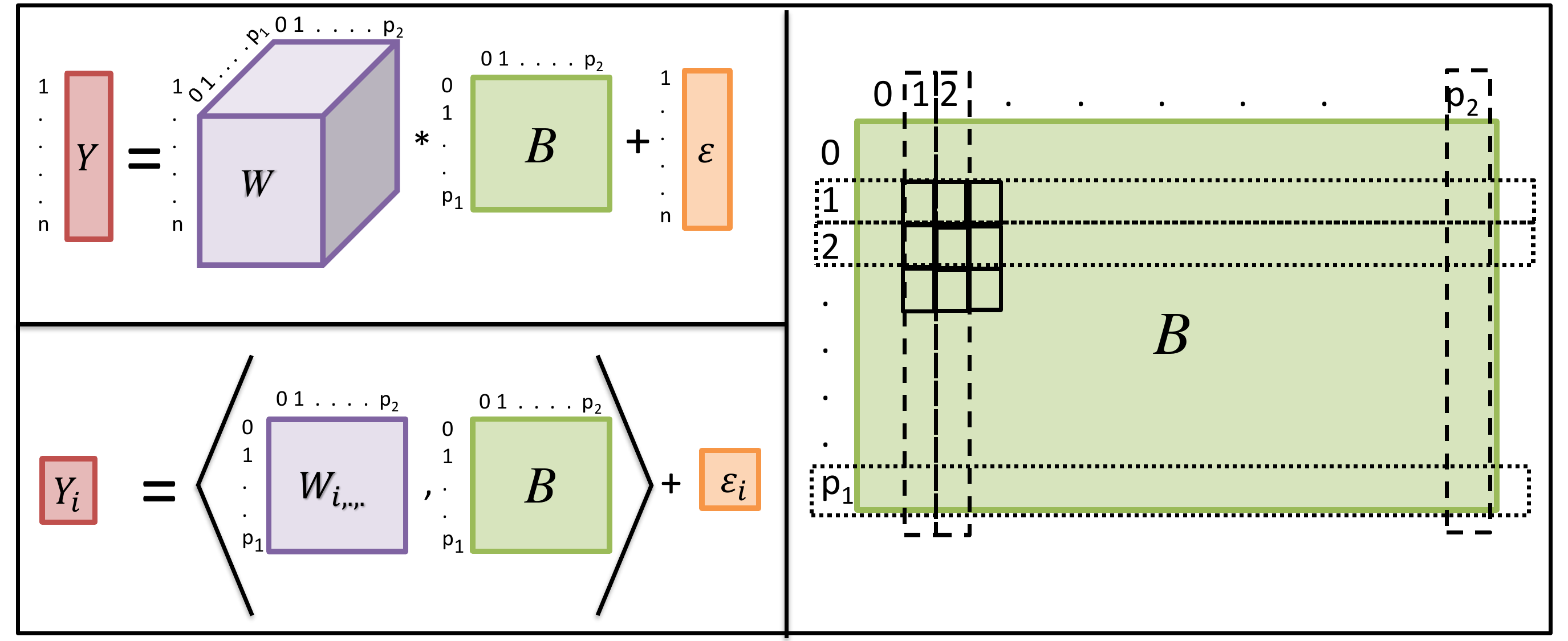}
\caption{ \emph{Left:}  The model (\ref{model}), for all $n$ observations \emph{(top)} and for the $i^{th}$ observation \emph{(bottom)}. The notation $\langle W_{i,\cdot,\cdot}, B \rangle$ denotes the inner product, $\sum_{j,k} W_{i,j,k} B_{j,k}$.  \emph{Right:} In (\ref{obj}), the $(p_1+1)\times (p_2+1)$ coefficient matrix $B$ is penalized by applying the $P_r$ and $P_c$ penalties to each of the $p_1$ rows (\protect\includegraphics[height=0.5em]{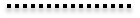}) and each of the $p_2$ columns (\protect\includegraphics[height=0.5em]{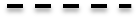}), respectively. The $\ell_1$ penalty is applied to each of the $p_1p_2$ interactions (\protect\includegraphics[height=0.5em]{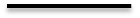}).}
\label{fig1}
\end{figure}

\subsection{Connections to Lasso \citep{tibshirani1996regression}}
\label{sec:compareLASSO}

The \emph{main effects lasso} can be viewed as a special case of (\ref{obj}) where $P_c$ and $P_r$ are $\ell_1$ penalties, 
\begin{equation}
\begin{split}
\underset{B\in \mathbb{R}^{(p_1+1)\times (p_2+1)}}{\text{minimize}} &\frac{1}{2n} \|y-W*B\|_2^2 +  \lambda_1 \sum_{j=1}^{p_1}\| B_{j,.}\|_1 + \lambda_2 \sum_{k=1}^{p_2}\|B_{.,k}\|_1 + \lambda_3 \|B_{-0,-0}\|_1,
\end{split}
\label{eqn:mainEffectsLasso}
\end{equation}
and where $\lambda_3$ is chosen sufficiently large as to shrink all of  the interaction terms to 0. In this case, the lasso penalties on the rows and columns are applied only to the main effects.

In contrast, if we take $\lambda_3 = 0$, $\lambda_1=\lambda_2 = \lambda$, and $P_c(b)=P_r(b) = |b_1|+1/2 \| b_{-1} \|_1$, where $b=(b_1, b_{-1}^T)^T$, then \eqref{obj} yields the \emph{all-pairs lasso}, which applies a lasso penalty to all main effects and all interactions. In this case, \eqref{obj} can be re-written more simply as 
\begin{equation}
\begin{split}
\underset{B\in \mathbb{R}^{(p_1+1)\times (p_2+1)}}{\text{minimize}} \frac{1}{2n} \|y-W*B\|_2^2 &+  \lambda \|B\|_1 .
\end{split}
\label{eqn:allPairsLasso}
\end{equation} 

However, our main interest in this paper is to develop a convex framework for modeling interactions that obeys strong heredity. Clearly, the all-pairs lasso does not satisfy strong heredity, and the main effects lasso does so only in a trivial way (by setting all interaction coefficient estimates to zero).

\subsection{\family \ with Strong Heredity} \label{subsec:strong}
We now consider three choices of $P_r$ and $P_c$ in \eqref{obj} that yield an estimator that obeys strong heredity. In Section~\ref{subsubsec:gl}, we consider the case where $P_r$ and $P_c$ are group lasso penalties. In Section~\ref{subsubsec:linf}, we consider the case where they are $\ell_\infty$ penalties. We consider a hybrid between an $\ell_1$ and an $\ell_\infty$ norm in Section~\ref{subsubsec:hiernet}. The unit norm balls corresponding to these three penalties are displayed in Figure~\ref{fig:norms}. 

\subsubsection{\family \ with an $\ell_2$ Penalty } 
\label{subsubsec:gl}

We first consider \eqref{obj} in the case where $P_r(b)=P_c(b) = \| b \|_2$, which we will refer to as \familylt.  The resulting optimization problem takes the form
\begin{equation}
\underset{B\in \mathbb{R}^{(p_1+1)\times (p_2+1)}}{\text{minimize}} \frac{1}{2n} \|y-W*B\|_2^2 +  \lambda_1 \sum_{j=1}^{p_1} \| B_{j,.} \|_2 + \lambda_2 \sum_{k=1}^{p_2} \| B_{.,k} \|_2 + \lambda_3 \|B_{-0,-0}\|_1.
\label{obj:l2}
\end{equation}
This formulation will induce strong heredity, in the sense that an interaction between $X_j$ and $X_k$ can have a non-zero coefficient estimate only if both of the corresponding main effects are non-zero. 

Problem \ref{obj:l2} is closely related to \vanish, an approach for non-linear interaction modeling  \citep{radchenko2010variable}. In fact, if we take $X=Z$ and assume that all main effects and interactions are scaled to have norm one in \eqref{obj:l2}, and consider the case of \vanish \ with only linear main effects and interactions, then \vanish \ and \eqref{obj:l2} coincide exactly.

  \citet{radchenko2010variable} attempt to solve the \vanish \ optimization problem via  block coordinate descent. However, due to non-separability of the groups, their algorithm is not guaranteed convergence to the global optimum. In contrast, the algorithm  in Section~\ref{sec:alg}  is guaranteed convergence to the global optimum of (\ref{obj})  for any convex penalty, and can be  extended to the case of generalized linear models.

\subsubsection{\family \ with an $\ell_\infty$  Penalty } 
\label{subsubsec:linf}

We now consider \eqref{obj} in the case where $P_r(b)=P_c(b) = \| b \|_\infty$; we refer to this in what follows as \familyli. We refer the reader to \citet{duchi2009efficient} for a discussion of the properties of the $\ell_\infty$ norm, and its merits relative to the $\ell_2$ norm in inducing group sparsity.  
In this case, \eqref{obj} takes the form
\begin{equation}
\underset{B\in \mathbb{R}^{(p_1+1)\times (p_2+1)}}{\text{minimize}} \frac{1}{2n} \|y-W*B\|_2^2 +  \lambda_1 \sum_{j=1}^{p_1} \| B_{j,.} \|_\infty + \lambda_2 \sum_{k=1}^{p_2} \| B_{.,k} \|_\infty + \lambda_3 \|B_{-0,-0}\|_1.
\label{obj:li}
\end{equation}
This formulation also induces strong heredity. 
\subsubsection{\family \ with a Hybrid $\ell_1$/$\ell_\infty$ Penalty}
\label{subsubsec:hiernet}

Finally, we consider \eqref{obj} with $P_r(b)=P_c(b)=\max(|b_1|, \| b_{-1} \|_1)$. In this case, \eqref{obj} takes the form
\begin{equation}
\small
\underset{B\in \mathbb{R}^{(p_1+1)\times (p_2+1)}}{\text{minimize}} \frac{1}{2n} \|y-W*B\|_2^2 +  \lambda_1 \sum_{j=1}^{p_1} \max( | B_{j,0} | , \| B_{j,-0} \|_1 ) + \lambda_2 \sum_{k=1}^{p_2}  \max( | B_{0,k} | , \| B_{-0,k} \|_1 ) + \lambda_3 \|B_{-0,-0}\|_1.
\label{obj:lhiernet}
\end{equation}
In the special case where $X=Z$, $\lambda_1 = \lambda_2 = \lambda$,  and $\lambda_3 = \lambda/2$,  \eqref{obj:lhiernet} is in fact equivalent to the \hiernet \ proposal of 
 \citet{bien2013lasso}.
   Details of this equivalence are given in \citet{bien2013lasso}.

 \citet{bien2013lasso} propose to solve \hiernet \ via an ADMM algorithm which applies a generalized gradient descent loop within each update. This leads to computational inefficiency, especially for large $p$.  In Section \ref{sec:alg}, we propose a simple, stand-alone ADMM algorithm for solving \eqref{obj}, which can be easily applied to solve \eqref{obj:lhiernet}, and consequently also the \hiernet \ optimization problem. 

Given its connection to \citet{bien2013lasso}, we refer to \eqref{obj:lhiernet} as \familyhiernet.

\subsubsection{Dual Norms}
\label{sec:dual}

\begin{figure}
\centering
\includegraphics[scale = 0.5]{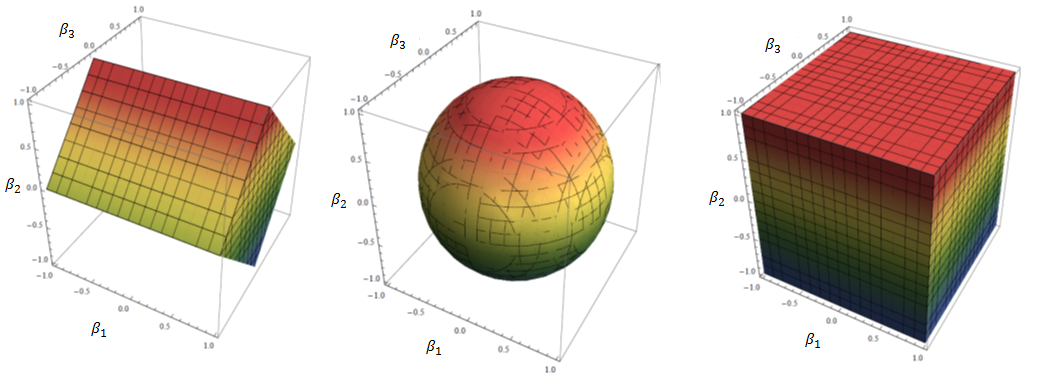}
\caption{A graphical representation of the region $P(\beta) \le 1$, where  $P(\beta)=\max \left( |\beta_1|, |\beta_2|+|\beta_3| \right)$ \emph{(left)}; $P(\beta) = \sqrt{\beta_1^2 + \beta_2^2 + \beta_3^2}$ \emph{(center)}; or $P(\beta) = \max(|\beta_1|, |\beta_2|, |\beta_3|)$ \emph{(right)}. }
\label{fig:norms}
\end{figure}

Here we further consider the  $l_2,\ l_{\infty}$ and $l_1/l_{\infty}$ hybrid penalties discussed in Sections~\ref{subsubsec:gl}-\ref{subsubsec:hiernet}. For an arbitrary penalty, the proximal operator is the solution to the optimization problem 
\begin{equation} 
\underset{\beta }{\text{minimize}}\ \   \frac{1}{2}\|y-\beta\|^2 + \lambda P(\beta).
\label{eqn:gen-prox}
\end{equation}
We begin by presenting a well-known lemma (see e.g. Proposition 1.1,  \citet{bach2011convex}).

\begin{lemma}
\label{lemma:dual-prox} 
Let $P(y)$ be a  norm of $y$ with dual norm $P_*(y) \equiv \max_z\  \{z^Ty: P(z) \le 1\} $. Then  $\hat{\beta} = 0$ solves (\ref{eqn:gen-prox}) if and only if $P_{*}(y) \leq  \lambda$.
\end{lemma}

It is well-known that the $\ell_2$ norm is its own dual norm, and that the $\ell_1$ norm is dual to the $\ell_\infty$ norm. We now derive the dual norm for the \familyhiernet\ penalty. This lemma is proven in Appendix~\ref{app:proofs}.  
\begin{lemma}
\label{lemma:hierNet-dual}
The dual norm of $P(\beta) = \max\{ |\beta_1|,\|\beta_{-1}\|_1 \}$ takes the form 
\begin{equation}
P_*(\beta)  = |\beta_1|+\|\beta_{-1}\|_{\infty}.
\label{eqn:hier_dual}
\end{equation}
\end{lemma}

Lemmas~\ref{lemma:dual-prox} and \ref{lemma:hierNet-dual} provide insight into the values of $y$ for which  all variables are shrunken to zero in \eqref{eqn:gen-prox}. The dual norm balls for the hybrid $\ell_1$/$\ell_\infty$, $\ell_2$, and $\ell_\infty$  norms are displayed in Figure~\ref{fig:dualnorms}. By Lemma~\ref{lemma:dual-prox}, any $y$ inside the dual norm ball leads to a zero solution of (\ref{eqn:gen-prox}). For the hybrid $\ell_1$/$\ell_\infty$ norm, the shape of the dual norm ball implies that  the first element of $y$ plays an outsize role in whether or not the coefficient vector is shrunken to zero. Consequently, the main effects play a larger role than the interactions in determining whether sparsity is induced. In contrast, for the $\ell_\infty$ and $\ell_2$ norms, the main effect and interactions play an equal role in determining whether the coefficients are shrunken to zero.

\begin{figure}
\centering
\includegraphics[scale = 0.6]{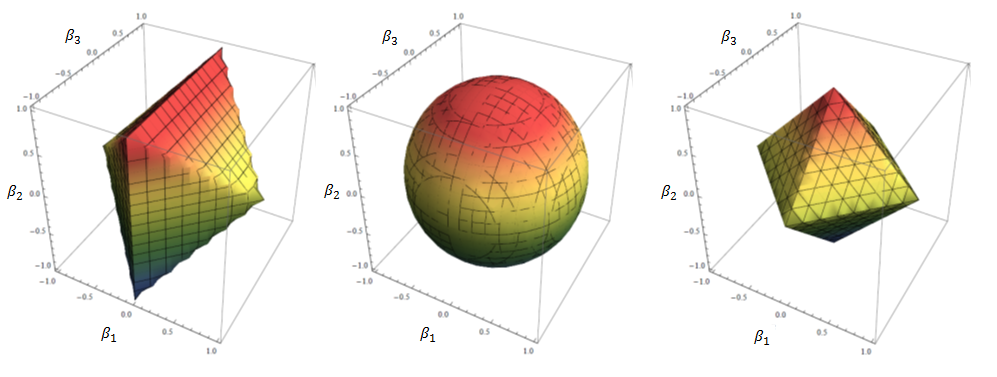}
\caption{A graphical representation of the region $P_*(\beta) \le 1$, where $P_*(\beta)$ is the dual norm for $P(\beta)=\max \left( |\beta_1|, |\beta_2|+|\beta_3| \right)$ \emph{(left)}; $P(\beta) = \sqrt{\beta_1^2 + \beta_2^2 + \beta_3^2}$ \emph{(center)}; or $P(\beta) = \max(|\beta_1|, |\beta_2|, |\beta_3|)$ \emph{(right)}. }
\label{fig:dualnorms}
\end{figure}

\subsection{Algorithm for Solving \family}
\label{sec:alg}
A step-by-step ADMM algorithm for solving \family\ is provided in Appendix~\ref{sec:fullAlg}.
Here, we present an overview of this algorithm. 
A gentle introduction to ADMM is provided in Appendix~\ref{sec:admm}. 

\subsubsection{ADMM Algorithm for Solving \family}
\label{sec:admm2}
We now develop an ADMM algorithm  to solve (\ref{obj}). We define the variable $\Theta = (D|E|F)$, with $D,E,F\in \mathbb{R}^{(p_1+1)\times (p_2+1)}$. That is, $\Theta$ is a $(p_1+1)\times 3(p_2+1)$ matrix, which we partition into $D,\ E$, and $F$ for convenience. Then (\ref{obj}) can be re-written as
\begin{eqnarray}
&&\underset{\parbox{1.5in}{ $B\in \mathbb{R}^{(p_1+1)\times (p_2+1)} ,\\ \Theta \in \mathbb{R}^{(p_1+1)\times 3(p_2+1)}$ } } { \text{minimize }} \left\{ \frac{1}{2n}\|{y}-{W}*B\|_2^2 +  \lambda_1 \sum_{j=1}^{p_1}P_r(D_{j,.})  + \lambda_2 \sum_{k=1}^{p_2}P_c(E_{.,k}) + \lambda_3 \|F_{-0,-0}\|_1 \right\} \nonumber \\
&&\mathrm{subject \; to\;} \qquad B(I_{(p_2+1)\times (p_2+1)} |I_{(p_2+1)\times (p_2+1)} | I_{(p_2+1)\times (p_2+1)} ) = \Theta.
\label{objadmm}
\end{eqnarray}
The augmented Lagrangian corresponding to \eqref{objadmm} takes the form
\begin{equation}
\small
\begin{split}
 L_{\rho} (B,\Theta, \Gamma  ) &= \frac{1}{2n} \|{y}-{W}*B\|_2^2 +\lambda_1 \sum_{j=1}^{p_1} P_r({D}_{j,.})+\lambda_2  \sum_{k=1}^{p_2} P_c({E}_{.,k})+ \lambda_3 \|{F}_{-0,-0}\|_1 \\
&+\text{trace} \left(\Gamma^T({B}({I}|{I}|{I})-{\Theta} )\right) + \rho/2 \| {B}({I}|{I}|{I})-{\Theta} \|_F^2, 
 \end{split}
 \label{auglag}
\end{equation}
where  $\Gamma $ is a $(p_1+1)\times 3(p_2+1)$-dimensional dual variable. For convenience, we partition $\Gamma$ as follows: $\Gamma = (\Gamma_1|\Gamma_2|\Gamma_3)$ where $\Gamma_i$ is a $(p_1+1)\times (p_2+1)$ matrix for $i=1,2,3$.  

The augmented Lagrangian  (\ref{auglag}) can be rewritten as
\begin{equation}
\begin{split}
L_{\rho} ({B,\Theta,\Gamma }) &=\frac{1}{2n} \|{y}-{W}*B\|_2^2 + \lambda_1 \sum_{j=1}^{p_1} P_r({D}_{j,.}) + \lambda_2 \sum_{k=1}^{p_2} P_c({E}_{.,k})+ \lambda \| {F}_{-0,-0}\|_1 \\
&+\langle\Gamma_1,{B}-{D}\rangle + \langle \Gamma_2,{B}-{E}\rangle +\langle \Gamma_3, {B}-{F}\rangle\\
&+ \rho/2 \| {B}-{D} \|_F^2+\rho/2 \| {B}-{E} \|_F^2 +\rho/2 \| {B}-{F} \|_F^2.
\end{split}
\label{auglag2}
\end{equation}

In order to develop an ADMM algorithm to solve \eqref{obj}, we must now simply figure out how to minimize \eqref{auglag2} with respect to $B$ with $\Theta$ held fixed, and how to minimize \eqref{auglag2} with respect to $\Theta$ with $B$ held fixed. 
 Minimizing \eqref{auglag2} with respect to $B$ 
 amounts simply to a least squares problem. In order to minimize \eqref{auglag2} with respect to $\Theta$, we note that \eqref{auglag2}  can simply be minimized with respect to $D$, $E$, and $F$ separately. 
Minimizing \eqref{auglag2} with respect to $F$ amounts simply to soft-thresholding \citep{friedman2007pathwise}. Minimizing \eqref{auglag2} with respect to $D$  or with respect to $E$ amounts to solving a problem that is equivalent to \eqref{eqn:gen-prox}. We consider that problem next.

Details of the ADMM algorithm for solving \eqref{obj} are given in Appendix~\ref{sec:fullAlg}.

\subsubsection{Solving (\ref{eqn:gen-prox}) for $\ell_2$, $\ell_\infty$, and Hybrid $\ell_1$/$\ell_\infty$  Penalties}
\label{sec:prox-all}

We saw in the previous section that the updates for $D$ and $E$ in the ADMM algorithm amount to solving the problem \eqref{eqn:gen-prox}. 
For $P(\beta) = \|\beta\|_2$,  (\ref{eqn:gen-prox}) amounts to soft-shrinkage \citep{simon2013sparse,yuan2006model}, for which a closed-form solution is available. For $P(\beta) = \|\beta\|_{\infty}$, an efficient algorithm was proposed by \citet{duchi2009efficient}. We now present an efficient algorithm for solving (\ref{eqn:gen-prox}) for  $P(\beta)  = \max\{ |\beta_1|, \|\beta_{-1}\|_{1}\}$.
\begin{lemma}
\label{lemma:hierNet-dualprob}
Let $\hat\beta$ denote the solution to (\ref{eqn:gen-prox}) with $P(\beta) = \max\{ |\beta_1|,\|\beta_{-1}\|_1 \}$. Then $\hat\beta=y - \hat{u}$, where $\hat{u}$ is the solution to 
\begin{equation}
\begin{split}
\underset{u\in \mathbb{R}^p,\ \lambda_1\in \mathbb{R}}{\mathrm{minimize}}  \ &\frac{1}{2}\|y-u\|^2 \\
\mathrm{subject\; to\; } &|u_1| \le \lambda_1,\ \|u_{-1}\|_{\infty} \le \lambda-\lambda_1,\ \ 0\le \lambda_1 \le \lambda.
\label{eqn:dual}
\end{split}
\end{equation}
\end{lemma}
We established in Section~\ref{sec:dual} that if $\lambda \geq |y_1|+\|y_{-1}\|_{\infty}$, then the solution to \eqref{eqn:gen-prox} is zero. 
Therefore, we now restrict our attention to the case $\lambda < |y_1|+\|y_{-1}\|_{\infty}$. For a fixed $\lambda_1 \in [0,\lambda]$, we can see by inspection that the solution to (\ref{eqn:dual}) is given by 
\begin{equation}
u_1(\lambda_1) = \left\{ \begin{array}{cc}
y_1 & |y_1| \le \lambda_1\\
\lambda_1 \mathrm{sgn}(y_1) & |y_1| > \lambda_1
\end{array} \right. \text{and } u_i(\lambda_1) = \left\{ \begin{array}{cc}
y_i & |y_i| \le \lambda - \lambda_1\\
(\lambda -\lambda_1) \mathrm{sgn}(y_i) & |y_i| > \lambda-\lambda_1
\end{array} \right. ,
\label{eqn:udef}
\end{equation}
for $i = 2,\ldots, p$. Thus, (\ref{eqn:dual}) is equivalent to the problem 
\begin{equation}
\begin{split}
\underset{\lambda_1\in [0,\lambda]}{\text{minimize}} & \ \frac{1}{2}\|y-u(\lambda_1)\|^2.
\label{eqn:dual-noconstraint}
\end{split}
\end{equation}

\begin{theorem}
\label{thm:hierNet}
Let $z$ denote the $(p-1)$-vector whose $i^{th}$ element is $\lambda - |y_{i+1}|$. Then the solution to problem (\ref{eqn:dual-noconstraint}) is given by 
\begin{equation}
\hat{\lambda}_1 = \left\{\begin{array}{cc}
\lambda &\mathrm{if}\; \min_j\left\{ \frac{|y_1|+ \sum_{i=1}^{j} z_{(i)} }{j+1} \right\}\ge\lambda\\
0 & \mathrm{if}\; \min_j \left\{ \frac{|y_1|+ \sum_{i=1}^{j} z_{(i)} }{j+1} \right\} \le 0\\
 \min_j \left\{ \frac{|y_1|+ \sum_{i=1}^{j} z_{(i)} }{j+1} \right\} & \mathrm{otherwise}
\end{array} \right. \ .
\label{thm:main-hierNet}
\end{equation}
\end{theorem}
Combining Theorem \ref{thm:hierNet} and Lemma \ref{lemma:hierNet-dualprob} gives us a solution for (\ref{eqn:gen-prox}) with the hybrid $\ell_1$/$\ell_\infty$ penalty. 
 Proofs are given in Appendix~\ref{app:proofs}.

\subsubsection{Convergence, Computational Complexity, and Timing Results}
\label{sec:computationAlg}
As mentioned in Section \ref{sec:admm}, ADMM's convergence to the global optimum is guaranteed for the convex, closed and proper objective function (\ref{obj}) \citep{boyd2011distributed}. The computational complexity of the algorithm depends on the form of the penalty functions used. 

The update for $B$ is typically the most computationally-demanding step of the ADMM algorithm for \eqref{obj}. As pointed out in 
Appendix~\ref{sec:fullAlg}, this can be done very efficiently. We perform the singular value decomposition for a $n\times (p_1+1)(p_2+1)$-dimensional matrix \emph{once}, given the data matrix $W$. Then,  in each iteration of the ADMM algorithm, the update for $B$ requires simply an efficient matrix inversion using the Woodbury matrix formula.

We now report timing results for our \verb|R|-language implementation of \family, available in the package \family\  on \verb|CRAN|, on an Intel\textregistered \  Xeon\textregistered \  E5-2620 processor. 
  We considered an example with $n=350$ and  $p_1=p_2=500$ (for a total of $251,000$ features). Using the parametrization  (\ref{eqn:obj-reparametrized}), running \familylt\  with $\alpha = 0.7$ and a grid of 10 $\lambda$ values takes a median time of 330 seconds, and running \familyli\ takes a median time of 416 seconds.

%In practice, \family\ can be implemented within reasonable time for $p_1$ and $p_2$ equal to $500$. Computational time can be further improved by a partial \verb|C++| implementation and parallelization. Details are provided in Appendix~\ref{app-sec:fullAlg}.

\subsection{Extension to Generalized Linear Models} 
\label{sec:Extension}

The \family \ optimization problem \eqref{obj} can be extended to the case of a general convex loss function $l(\cdot)$, 
\begin{equation}
\begin{split}
 \underset{B\in \mathbb{R}^{(p_1+1)\times (p_2+1)}}{\text{minimize}} &\frac{1}{n} l(B) +   \lambda_1 \sum_{j=1}^{p_1}P_r( B_{j,.}) + \lambda_2 \sum_{k=1}^{p_2}P_c (B_{.,k}) + \lambda_3 \|B_{-0,-0}\|_1. 
\end{split}
\label{objglm}
\end{equation}
For instance, in the case of a binary response variable $y$, we could take $l$ to be the negative log likelihood under a binomial model. Then \eqref{objglm} corresponds to a penalized logistic regression problem with interactions.  
An ADMM algorithm for \eqref{objglm} can be derived just as in Section~\ref{sec:admm2}, with a  modification to the update for $B$.  
This is discussed in Appendix~\ref{app:alg-logistic}. 

\subsection{Uniqueness of the \family\ Solution}
\label{sec:identifiability}

The \family \ optimization problem \eqref{obj} is convex, and the algorithm presented in Section~\ref{sec:alg} is guaranteed to yield a solution that achieves the global minimum. But \eqref{obj} is not strictly convex: this means that the solution might not be unique, in the sense that more than one value of $B$ might achieve the global minimum.  However, uniqueness of the  \emph{fitted values} resulting from \eqref{obj} is straightforward.  This is formalized in the following lemma. The proof is as in  Lemma 1(ii) of \citet{tibshirani2013lasso}. 
\begin{lemma}
For a convex penalty function $P(\cdot)$, let $\hat{B}$ denote the solution to the problem 
\begin{equation}
\underset{B\in \mathbb{R}^{(p_1+1)\times(p_2+1)} }{\text{minimize}} \  \frac{1}{2n}\|y-W*B\|^2+P(B).
\label{eqn:identifiability}
\end{equation}
The fitted values $W*\hat{B}$ are unique.
\end{lemma}

\section{Degrees of Freedom}
\label{sec:dof}

\subsection{Review of Degrees of Freedom}

Consider the linear model $y = X \beta + \epsilon$, with fixed $X$, and   
$\epsilon \sim \mathcal{N}_n(0, \sigma^2 \bs{I}_n)$. Then the degrees of freedom of a model-fitting procedure is defined as \citep{stein1981estimation,efron1986biased}
\begin{equation}
\df = \frac{1}{\sigma^2} \sum_{i=1}^{n} \text{Cov}(y_i,\hat{y}_i),
\label{eqn:real.df}
\end{equation}
where $\hat{y}_i$ are the fitted response values. If certain conditions hold, then  
\begin{align}
\df &= E\left[\sum_{i=1}^{n} \frac{\partial \hat{y}_i}{\partial y_i} \right]. 
 \label{eqn:df.est}
\end{align} 
Therefore, $\sum_{i=1}^{n} \frac{\partial \hat{y}_i}{\partial y_i} $ is an unbiased estimator for the degrees of freedom of the model-fitting procedure.

Before presenting the main results of this section, we state a useful lemma.
\begin{lemma}
\label{lemma:hessian-lq}
Given a vector $x\in \mathbb{R}^p$, and an even positive integer $q$, 
\begin{equation}
\frac{d^2 \|x\|_q }{dx^2}  =  (q-1) \diag \left[ \left( \frac{x }{\|x\|_q} \right)^{q-2}\right] \times \left[ \frac{I}{\|x\|_q} - \frac{x(x^T)^{q-1}  }{\| x\|_q^{q+1} } \right],
\label{eqn:hessian-lq}
\end{equation}
where $\diag(x)$ is the diagonal matrix with $x$ on the diagonal, and $(x)^{q}$ denotes the element-wise exponentiation of the vector $x$. 
\end{lemma}

\subsection{Degrees of Freedom for  a Penalized Regression Problem }

We now consider  the degrees of freedom of the estimator that solves the problem 
\begin{equation}
\underset{\beta\in \mathbb{R}^p}{\text{minimize}} \ \frac{1}{2} \|y - X\beta \|_2^2 + \sum_{d} \lambda_d P_d(A_d\beta),
\label{eqn:df-problem}
\end{equation}
where $P_d(\cdot)$ is an $\ell_q$ norm for a positive $q$, and $A_d$ is a $p\times p$ diagonal matrix with ones and zeros on the main diagonal. We define the active set 
to be $\mathcal{A} = \{j: \hat{\beta}_j \not= 0 \}$, the set of non-zero coefficient estimates. Let ${\hat{\beta}}_{ \mathcal{A}}$ denote the coefficients of the active set, and let $X_{\mathcal{A}}$ denote the matrix with columns corresponding to elements of the active set. Furthermore, we define $A_d^{\mathcal{A}}$ to be the sub-matrix of $A_d$ with rows and columns in $\mathcal{A}$.

\begin{claim}
\label{claim:df}
An unbiased estimator of the degrees of freedom of $\hat\beta$, the solution to (\ref{eqn:df-problem}), is given by 
\begin{equation}
\widehat{\df} = \mathrm{trace} \left(  X_{\mathcal{A}}\left[ X^T_{\mathcal{A}} X_{\mathcal{A}} + \sum_d \lambda_d \left( A_d^{\mathcal{A}}\right)^T  \ddot{P}_d(A_d^{\mathcal{A}} \hat{\beta}_{\mathcal{A}}) \left(A_d^{\mathcal{A}}\right) \right]^{-1} X_{ \mathcal{A}}^T \right),
\label{eqn:df-est-main}
\end{equation}
where $\ddot{P}_d(\cdot)$ is the  Hessian of the function $P_d(\cdot)$, and where $\mathcal{A}$ is the active set. 
\end{claim} 
The derivation for Claim~\ref{claim:df} is outlined in Appendix \ref{app:dof}.

\subsection{Degrees of Freedom for  \family}

In this section we present estimates for the degrees of freedom of \familylt\ and \familyli. An estimate of the degrees of freedom of \familyhiernet\  is given in \citet{bien2013lasso}.

\subsubsection{\familylt}
We write \familylt\ in the form of (\ref{eqn:df-problem}),
\begin{equation}
\frac{1}{2} \|y-\wt{W}\wt{B}\|_2^2 +  n\lambda_1 \sum_{j=1}^{p_1} \| A_j\wt{B} \|_2 + n\lambda_2 \sum_{k=p_1+1}^{p_1+p_2} \| A_k\wt{B} \|_2 + n\lambda_3 \|A_I\wt{B}\|_1,
\label{eqn:familylt-df-form}
\end{equation}
where $\wt{B}$ is the vectorized version of $B$, and $\wt{W}$ is the $n\times (p_1+1)(p_2+1)$-dimensional matrix version of $W$. We apply Claim \ref{claim:df}  in order to obtain an unbiased estimate for \familylt:
\begin{equation} \label{eqn:df-est-l2} \footnotesize
\widehat{\df}_{\ell_2} = 
\mathrm{trace} \left(  \wt{W}_{\mathcal{A}} \left[ \wt{W}^T_{\mathcal{A}} \wt{W}_{\mathcal{A}} + n\lambda_1 \sum_{j=1}^{p_1} 
 (A_j^{\mathcal{A}})^T \left[ \ddot{P}(A_j^{\mathcal{A}} \hat{\wt{B}}_{\mathcal{A}} )  \right]  (A_j^{\mathcal{A}}) 
  +  n\lambda_2 \sum_{k=p_1+1}^{p_1+p_2}    
 (A_k^{\mathcal{A}})^T \left[  \ddot{P}(A_k^{\mathcal{A}} \hat{\wt{B}}_{\mathcal{A}} )  \right]  (A_k^{\mathcal{A}})  \right]^{-1} \wt{W}^T_{\mathcal{A}} \right),
\end{equation}
where $\ddot{P}(v_0)=\left. \frac{d^2 \|  v \|_2}{dv^2} \right|_{v=v_0}$ is of the form given in Lemma \ref{lemma:hessian-lq}.

\subsubsection{\familyli\ }
\label{sec:familyli-df-lq}

The $\ell_{\infty}$ norm is not differentiable, and thus we cannot apply Claim \ref{claim:df} directly. Instead, we make use of the fact that  $\lim\limits_{q\to \infty} \|\beta\|_q = \|\beta\|_{\infty}$ in order to  apply Claim~\ref{claim:df} to a modified version of \familyli\, in which the $\ell_\infty$ norm is replaced with an $\ell_q$ norm for a very large value of $q$. This yields the estimator 
\begin{equation} \label{eqn:df-est-lq} \footnotesize
\widehat{\df}_{\ell_\infty} = 
\mathrm{trace} \left(  \wt{W}_{\mathcal{A}} \left[ \wt{W}^T_{\mathcal{A}} \wt{W}_{\mathcal{A}} + n\lambda_1 \sum_{j=1}^{p_1} 
 (A_j^{\mathcal{A}})^T \left[ \ddot{P}(A_j^{\mathcal{A}} \hat{\wt{B}}_{\mathcal{A}} )  \right]  (A_j^{\mathcal{A}}) 
  +  n\lambda_2 \sum_{k=p_1+1}^{p_1+p_2}    
 (A_k^{\mathcal{A}})^T \left[  \ddot{P}(A_k^{\mathcal{A}} \hat{\wt{B}}_{\mathcal{A}} )  \right]  (A_k^{\mathcal{A}})  \right]^{-1} \wt{W}^T_{\mathcal{A}} \right),
\end{equation}
where $\ddot{P}(v_0)=\left. \frac{d^2 \|  v \|_q}{dv^2} \right|_{v=v_0}$ is of the form given in Lemma \ref{lemma:hessian-lq}.
 We use $q=500$ in Section~\ref{sec:df-simulations}.

\subsection{Numerical Results}
\label{sec:df-simulations}

We now consider the numerical performance of our estimates of the degrees of freedom of \family\ in a simple simulation setting. We use a fixed design matrix $X$, with $n=100$ rows and $p = 10$ main effects, and we let $X=Z$. We randomly selected 15 true interaction terms. We generated 100 different response vectors $y^{(1)},\ldots,y^{(100)}$ using independent Gaussian noise. We computed the  true degrees of freedom as well as the estimated degrees of freedom from \eqref{eqn:df-est-l2} and \eqref{eqn:df-est-lq}, averaged over the 100 simulated data sets. In Figure~\ref{fig:dof}, we see almost perfect agreement between the true and estimated degrees of freedom.
\begin{figure}
\centering
\includegraphics[scale = 0.35]{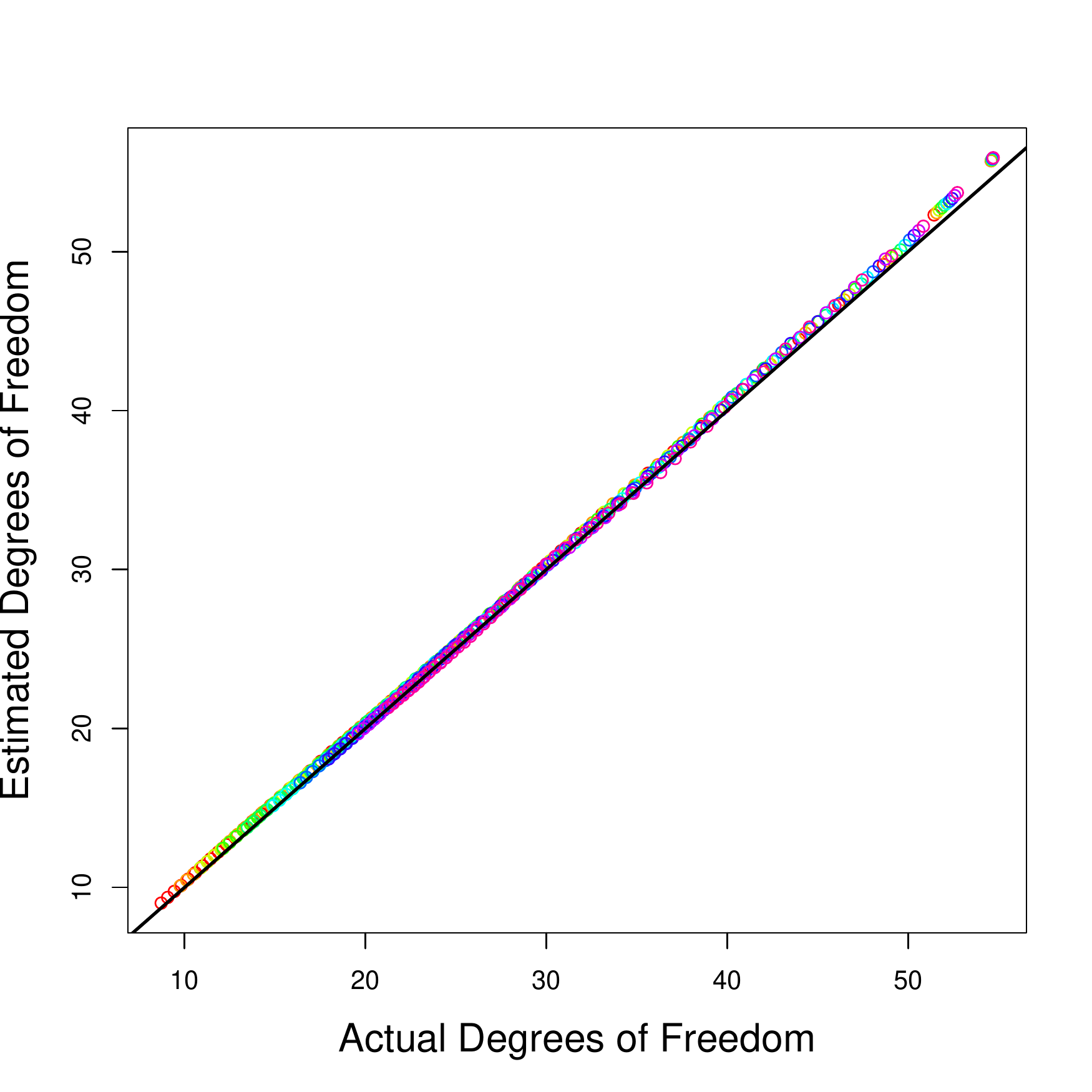}
\includegraphics[scale = 0.35]{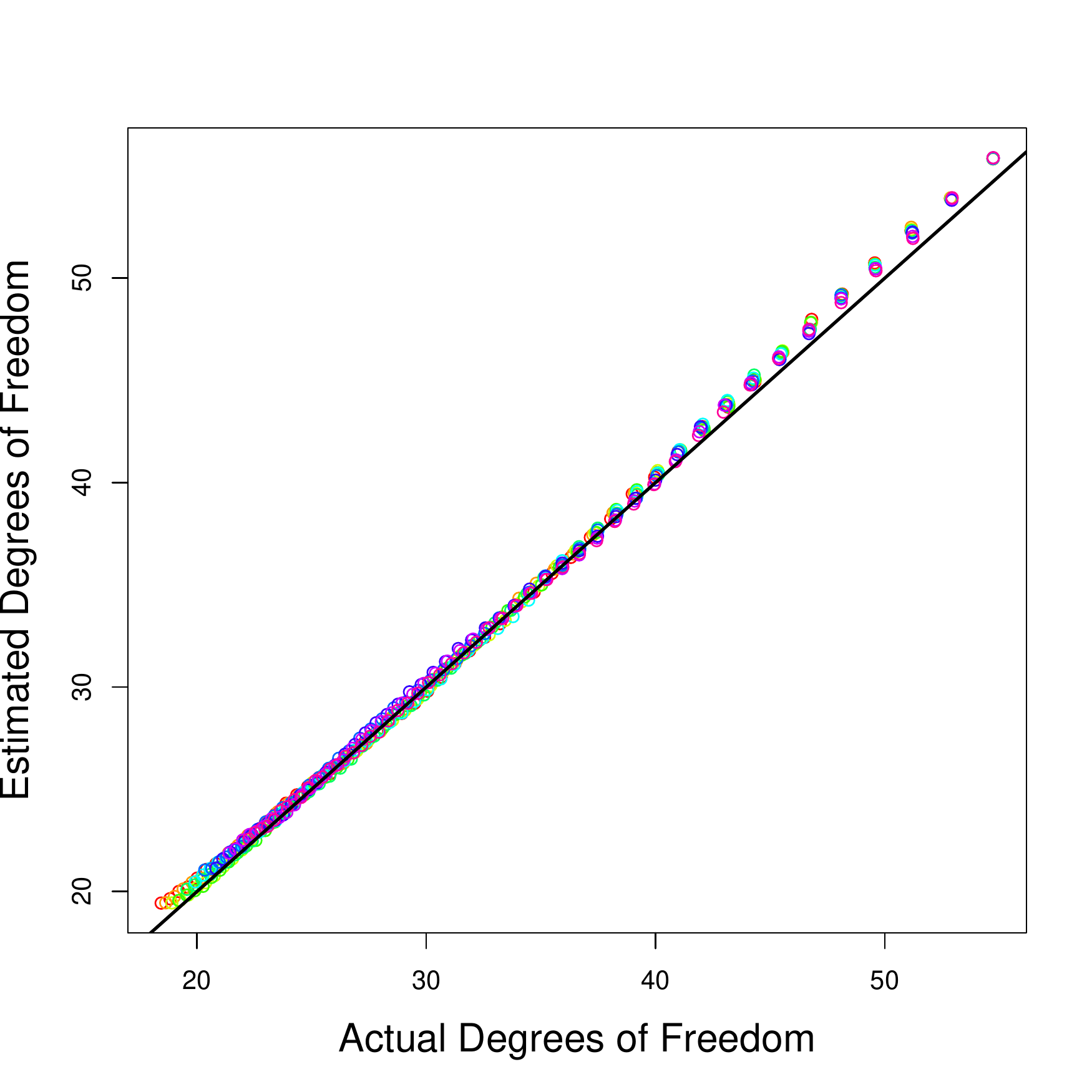}
\caption{The estimated degrees of freedom as a function of the actual degrees of freedom, for \emph{(Left:)} \familylt\  and \emph{(Right:)} \familyli . To estimate the degrees of freedom for \familyli, we  used $q = 500$ in \eqref{eqn:df-est-lq}.  Several values of  $\alpha$ in were used in the \family\ optimization problem (using the reparametrization in \eqref{eqn:obj-reparametrized}); each is shown in a different color. Each point corresponds to a different value of $\lambda$ in the \family\ optimization problem.} 
\label{fig:dof}
\end{figure}

\section{Extension to Weak Heredity} \label{sec:weak}

We now consider a modification to the \family\ optimization problem, \eqref{obj},  that imposes weak heredity.  
  We assume that the main effects, interactions, and response have been centered to have mean zero.
 
In order to enforce weak heredity, we take an approach motivated by the latent overlap group lasso of \citet{jacob2009group}. We let $W^X$ denote the $n \times p_1 \times (p_2+1)$ array defined as follows: for $i\in \{1,\ldots,n\},\  j\in \{1,\ldots, p_1\},\  k\in \{0,\ldots,p_2\}, $ 
\begin{equation}
W^X_{i,j,k} = 
\begin{cases}
X_{i,j}Z_{i,k} & \mbox{ for }  k\not=0\\
X_{i,j} & \mbox{ for }  k=0\\
\end{cases}. 
\label{eqn:Wx-mat}
\end{equation}
 We let $W^Z$ denote the $n \times (p_1+1) \times p_2$ array defined in an analogous way.  
  We take $B^X$ to be a $p_1\times (p_2+1)$ matrix, and   $B^Z$ to be a $(p_1+1) \times p_2$ matrix.

We propose to solve the optimization problem 
\begin{equation}
\begin{split}
\underset{ \parbox{1.4in}{ $B^X\in \mathbb{R}^{p_1\times (p_2+1)}\\ B^{Z} \in \mathbb{R}^{(p_1+1)\times p_2} $ }  }{\text{minimize}} 
\  &\frac{1}{2n} \left\|y - W^X*B^X- W^Z*B^Z\right\|_2^2\\
&+ \lambda_1 \sum_{j=1}^{p_1} P_r (B^X_{j,.} ) + \lambda_2 \sum_{k=1}^{p_2} P_c (B^Z_{.,k} )
+ \lambda_3(\|B^X_{.,-0}\|_1+ \|B^Z_{-0,.}\|_1).
\end{split}
\label{eqn:weak}
\end{equation} 
 Then the coefficient for the $j^{th}$ main effect of $X$ is ${B}^X_{j,0}$, the coefficient for the $k^{th}$ main effect of $Z$ is ${B}^Z_{0,k}$, and the coefficient for the $(j,k)$ interaction is ${B}^X_{j,k}+{B}^Z_{j,k}$. If we take $P_r$ and $P_c$ to be either $\ell_2$, $\ell_\infty$, or hybrid $\ell_1$/$\ell_\infty$ penalties, then  \eqref{eqn:weak} imposes weak heredity:  if the $k$th column of $B_Z$ has a zero estimate, then the $(j,k)^{th}$ interaction coefficient estimate need not be zero. However, if the $j^{th}$ row of $B_X$ and the $k^{th}$ column of $B_Z$ have zero estimates, then the $(j,k)^{th}$ interaction coefficient estimate  is zero. 

Problem \eqref{eqn:weak} can be solved using an ADMM algorithm similar to that of Section~\ref{sec:alg}. 
 Since the focus of this paper is on enforcing strong heredity, we leave the details of an algorithm for \eqref{eqn:weak}, as well as a careful numerical study, to future work.

\section{Simulation Study} \label{sec:SimStudy}

We compare the performance of \familylt \ and \familyli \ to the all-pairs lasso (\APL), the \texttt{hierNet} proposal of \citet{bien2013lasso}, and the \texttt{glinternet} proposal of \citet{lim2013learning}. \APL \ can be performed using the \verb=glmnet= \verb=R= package, and  \hiernet \ and \glinternet \ are implemented in \verb|R| packages available on \verb|CRAN|. 
%~(Recall that  \hiernet \ \citep{bien2013lasso} was shown in Section~\ref{subsubsec:hiernet} to be equivalent to \familyhiernet. Therefore, in principle, it should not matter whether we use \citet{bien2013lasso}'s implementation of \hiernet, or our own implementation based on the  ADMM algorithm presented in Section~\ref{sec:alg}. However, in order to ensure fair treatment of the \hiernet\ proposal, we use the implementation of \citet{bien2013lasso} in all of our numerical results.)
  We also include the oracle model \citep{fan2001variable} --- an unpenalized model that uses only the main effects and interactions that are non-zero in the true model --- in our comparisons. 
%Finally, we consider the \iFORM\ forward selection proposal of \citet{hao2014interaction}. 

The forward selection proposal of \citet{hao2014interaction}, \iFORM, is a fast screening approach for detecting interactions in  ultra-high dimensional data. \iFORM\ is intended for the setting in which the true model is extremely sparse. In our simulation setting, we consider moderately sparse models, which fails to highlight the advantages of \iFORM. Thus, we do not include results for  \iFORM\ in our simulation study.

To facilitate comparison with \hiernet\ and \glinternet, which require $X=Z$, we take $X=Z$ in our simulation study. Similar empirical results are obtained in simulations with $X \neq Z$; results are omitted due to space constraints.

We consider squared error loss in Section~\ref{sec:squared}, and logistic regression loss in Section~\ref{sec:SimStudyGlm}. 
\subsection{Squared Error Loss} \label{sec:squared}

\subsubsection{Simulation Set-up}
\label{sec:datagen}

We created a coefficient matrix $B$, with $p=30$ main effects and ${p \choose 2}=435$ interactions, for a total of $465$ features. The first 10 main effects have non-zero coefficients, assigned uniformly from the set $\{-5,-4,\ldots,-1,1, \ldots,5\}$. The remaining main effects' coefficients equal zero. We consider three simulation settings, in which  we randomly select 15, 30 or 45 non-zero interaction coefficients, chosen to obey  strong heredity.  The values for the non-zero coefficients were selected uniformly from the set $\{ -10, -8,\ldots, -2, 2,\ldots, 8, 10 \}$. Figure \ref{simB} displays $B$ in each of the three simulation settings.

We generated a training set, a test set, and a validation set, each consisting of 300 observations.  Each observation of $X=Z$ was generated independently from a $\mathcal{N}_p(0,I)$ distribution;  $W$ was then constructed according to \eqref{wmat}.  For each observation we generated an  independent Gaussian noise term, with variance adjusted to maintain a signal-to-noise ratio of approximately 2.5 to 3.5.  Finally, for each observation,  a response was generated according to \eqref{arraynot}.

We applied \texttt{glinternet} and \texttt{hierNet} for 50 different values of the tuning parameters. For convenience, given that $X=Z$, we reparametrized the \family \ optimization problem \eqref{obj} as 
\begin{equation}
\begin{split}
\underset{B\in \mathbb{R}^{(p+1)\times (p+1)}}{\text{minimize}}\  \frac{1}{2n} \|y-W*B\|_2^2 &+  (1-\alpha)\lambda \sqrt{p} \sum_{j=1}^{p} P_r( B_{j,.}) + (1-\alpha)\lambda \sqrt{p} \sum_{k=1}^{p}  P_c (B_{.,k})\\
 &+ \alpha\lambda \|B_{-0,-0}\|_1.
\end{split}
\label{eqn:obj-reparametrized}
\end{equation}
We applied \familylt \ and \familyli \ over a $10\times 50$ grid of $(\alpha, \lambda)$ values, with $\alpha \in (0,1)$ and $\lambda$ chosen to give a suitable range of sparsity. 

\begin{figure}
\centering
\includegraphics[scale = 0.25]{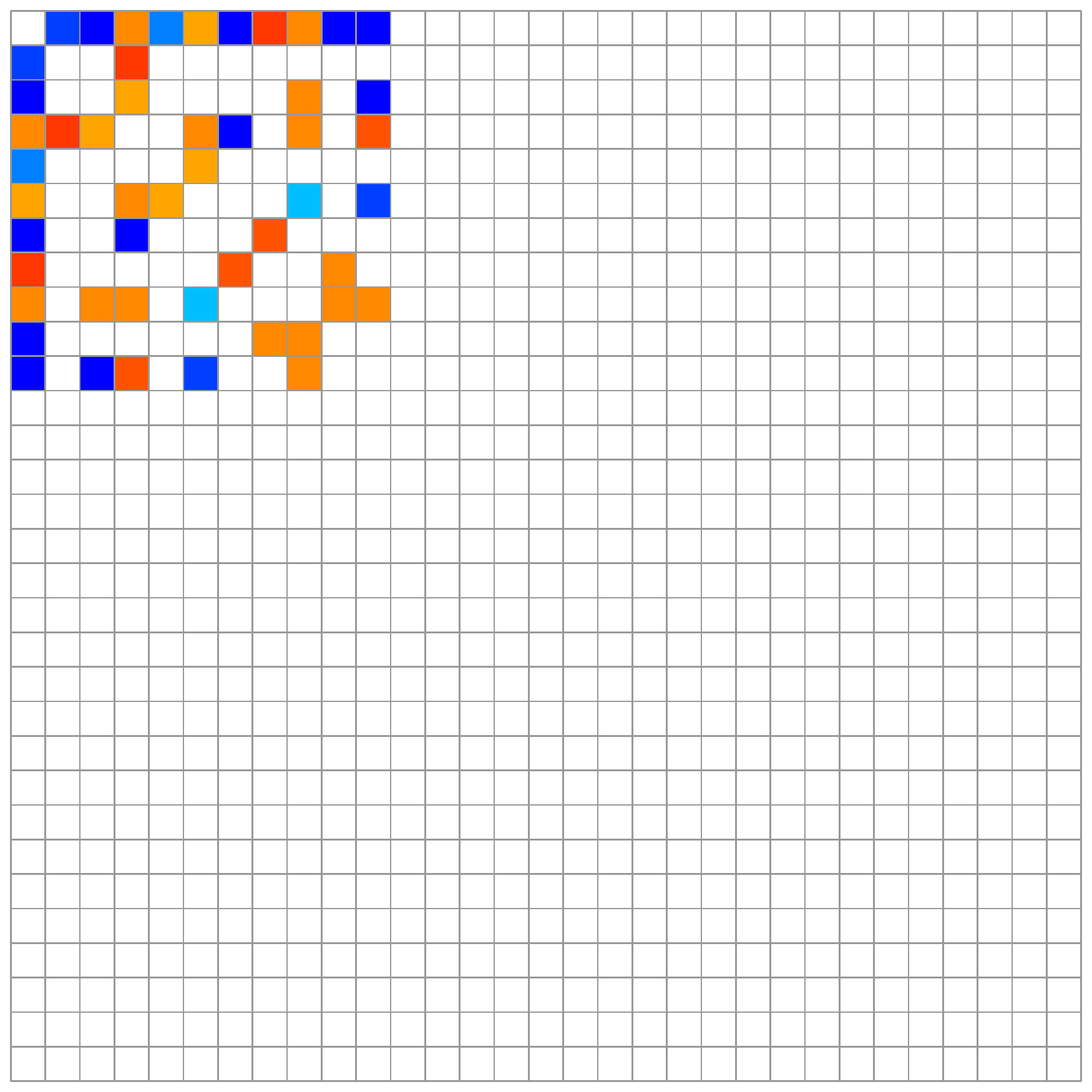}
\includegraphics[scale = 0.25]{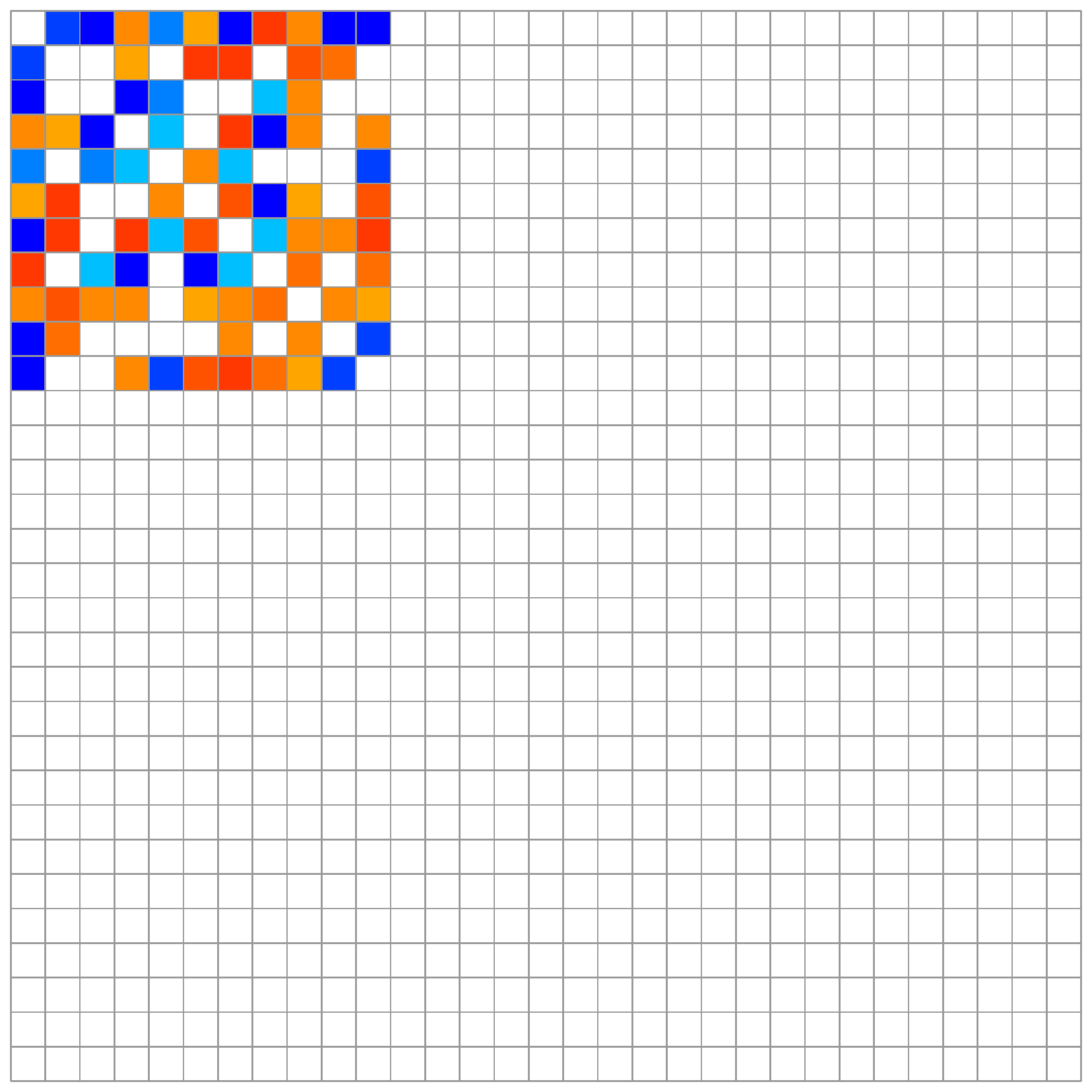}
\includegraphics[scale = 0.25]{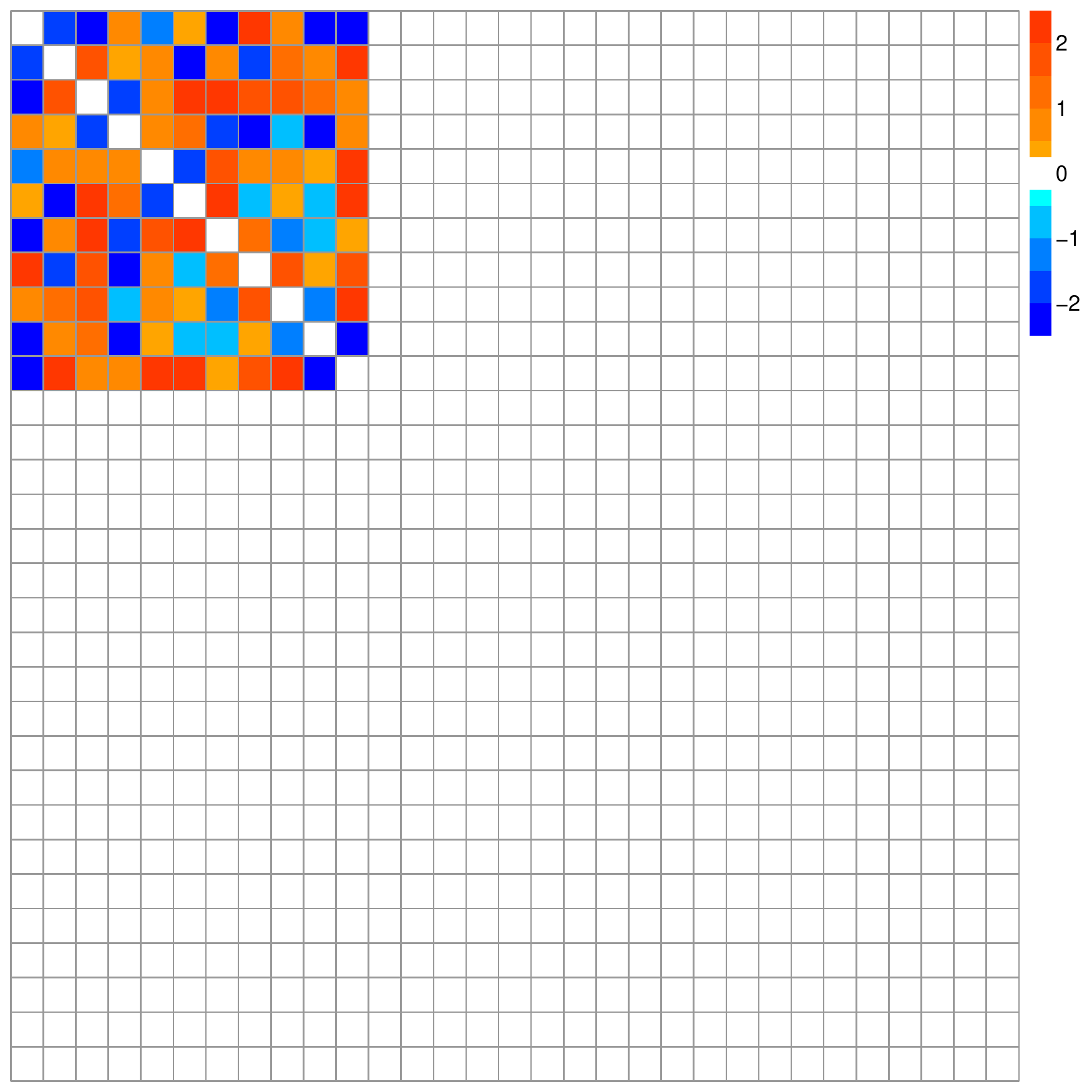}
\caption{ For the simulation study in Section~\ref{sec:SimStudy}, the heatmap of the matrix $B$ is displayed in the case of 15 (\emph{left}), 30 (\emph{center}), and 45 (\emph{right}) non-zero interactions. The first row and column of each heatmap represent the main effects. }
\label{simB}
\end{figure}

In principle, many methods are available for selecting the tuning parameters $\alpha$ and $\lambda$. These include Bayesian information criterion,  generalized cross-validation, and others.  Because we do not have an estimator for the degrees of freedom of the \glinternet\ estimator, we opted to use a training/test/validation set approach. 
In greater detail, we fit each method to the training set, selected tuning parameters based on sum of squared residuals (SSR) on the test set, and then reported the SSR for that choice of tuning parameters on the validation set.

It is well-known that penalized regression techniques tend to yield models with over-shrunken coefficient estimates  \citep{hastie2009elements,fan2001variable}. To overcome this problem, we obtained \emph{relaxed} versions of \familylt, \familyli, \hiernet, and \glinternet, by refitting an unpenalized least squares model to the set of coefficients that are non-zero in the penalized fitted model \citep{meinshausen2007relaxed,radchenko2010variable}.

We also considered generating the observations of $X$ from a $\mathcal{N}_p(0,\Sigma)$ distribution, where $\Sigma$ was  an autoregressive  or an exchangeable covariance matrix. We found that the choice of covariance matrix $\Sigma$ led to little qualitative difference in the results. Therefore, we display only results for $\Sigma=I$ in  Section~\ref{sec:Results}.
 
\subsubsection{Results}
\label{sec:Results}

The left panel of Figure \ref{fig:result} displays ROC curves for \familyli, \familylt, \hiernet, \glinternet, and \APL. 
 These results indicate that \familylt\ outperforms all other methods in terms of variable selection, especially as the number of non-zero interaction coefficients increases. When there are 45 non-zero interactions, \familyli\ outperforms \glinternet, \hiernet, and \APL.

The right panel of Figure \ref{fig:result} displays the test set SSR for all methods, as the tuning parameters are varied. We observe that relaxation leads to improvement for each method: it yields a much sparser model for a given value of the test error. This is not surprising, since the relaxation  alleviates some of the over-shrinkage induced by the application of multiple convex penalties. The results further indicate that when relaxation is applied, \familylt\ performs the best, followed by \familyli \ and then the other competitors. We once again observe that the improvement of \familylt\ and \familyli \ over the competitors increases as the number of non-zero interaction coefficients increases.

Interestingly, the right-hand panel of Figure~\ref{fig:result} indicates that though \familylt \ performs the best when relaxation is performed, it performs quite poorly when relaxation is not performed, in that the model with smallest test set SSR contains far too many non-zero interactions. This is consistent with the remark in \citet{radchenko2010variable} regarding over-shrinkage of coefficient estimates.

In Table 1, we present results on the validation set for the model that was fit on the training set using the tuning parameters selected on the test set, as described in Section~\ref{sec:datagen}. We see that \familylt\ and \familyli\ outperform the competitors in terms of SSR, false discovery rate, and true positive rate, especially when relaxation is performed.

%COMMENTED OUT: 2015/06/14 ASAD HARIS
%We now summarize timing results for our \verb=R=-language implementations of  \familylt\ and \familyli, run on an Intel\textregistered \  Xeon\textregistered \  E5-2620 processor.  On a $10\times 50$ grid of $(\alpha, \lambda)$ values, \familylt\ took an average of 350 seconds, and  \familyli\ took an average of 550 seconds. The computation time of \familyli\ is higher than that of \familylt\ because computing the prox operator of the former is more expensive.

\begin{figure}
\centering
15 Non-Zero Interactions \\ 
\includegraphics[scale = 0.33]{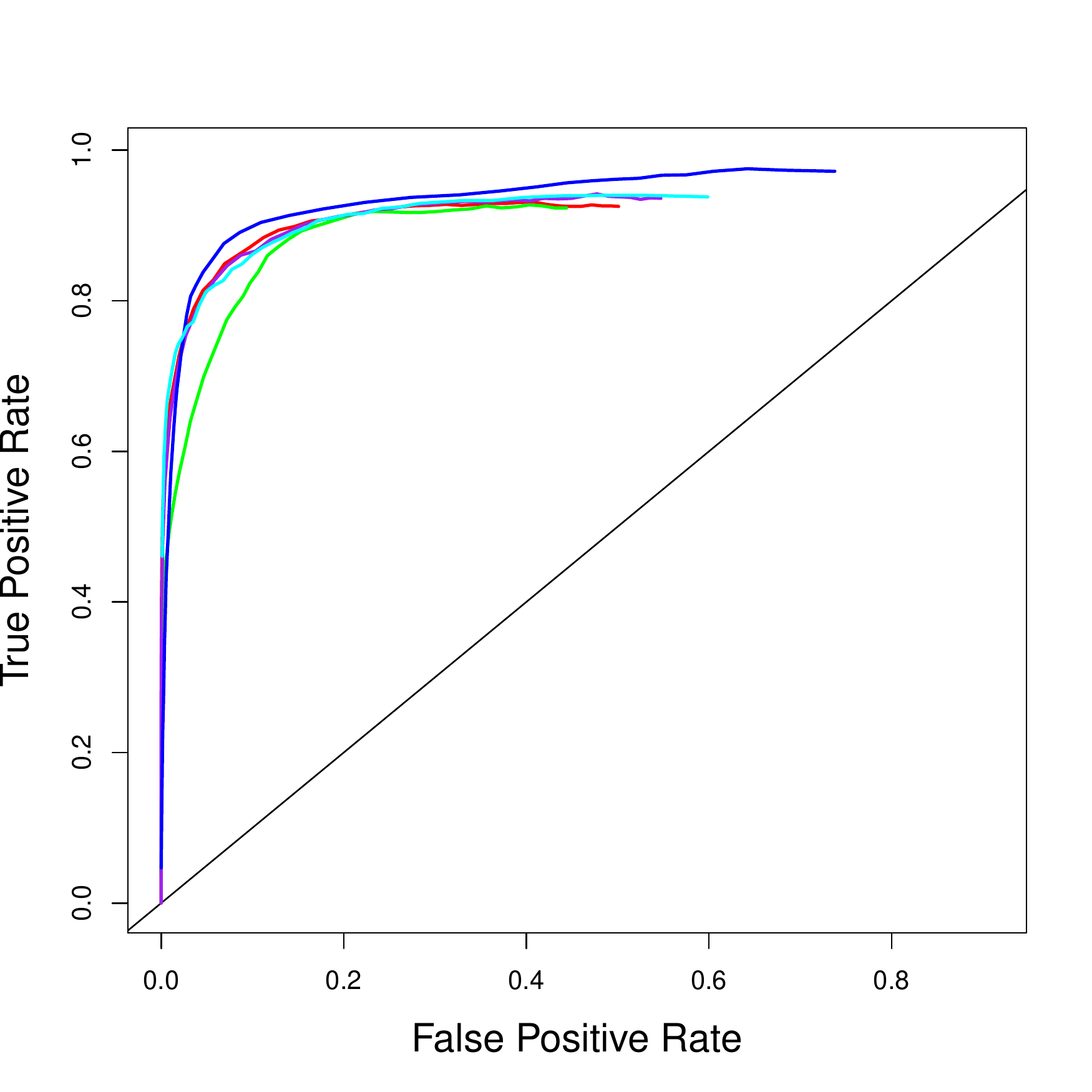}
\includegraphics[scale = 0.33]{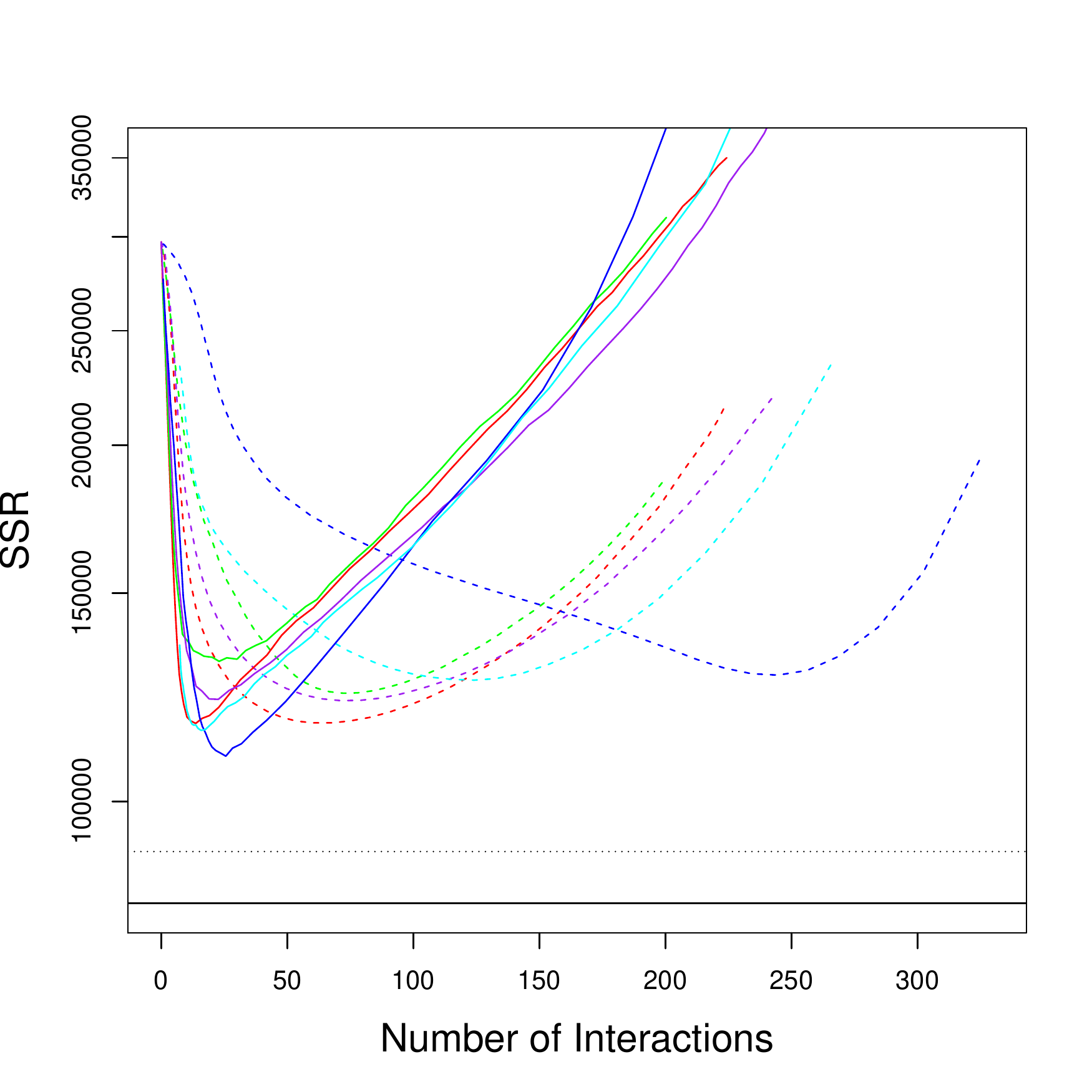} \\
30 Non-Zero Interactions \\ 
\includegraphics[scale = 0.33]{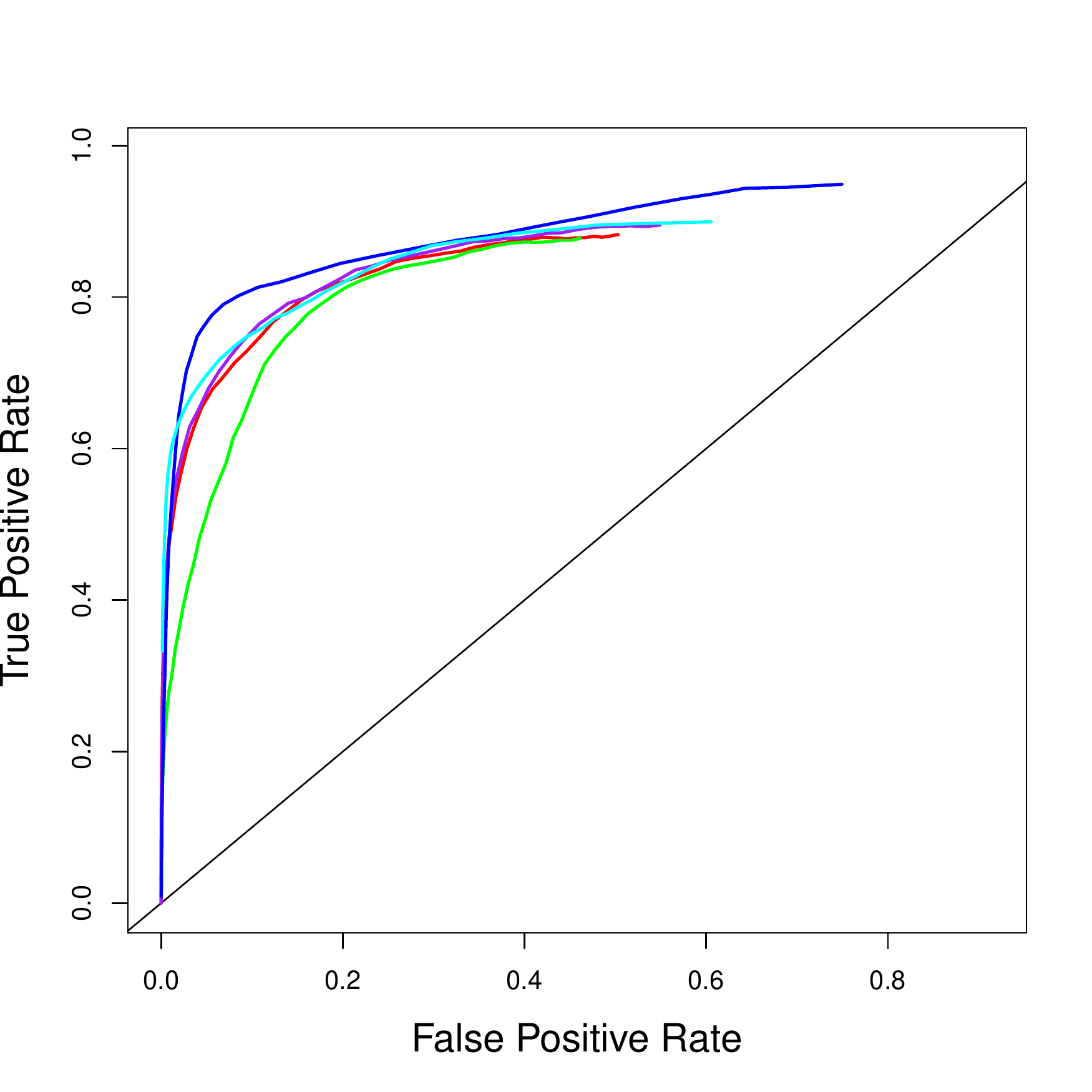} 
\includegraphics[scale = 0.33]{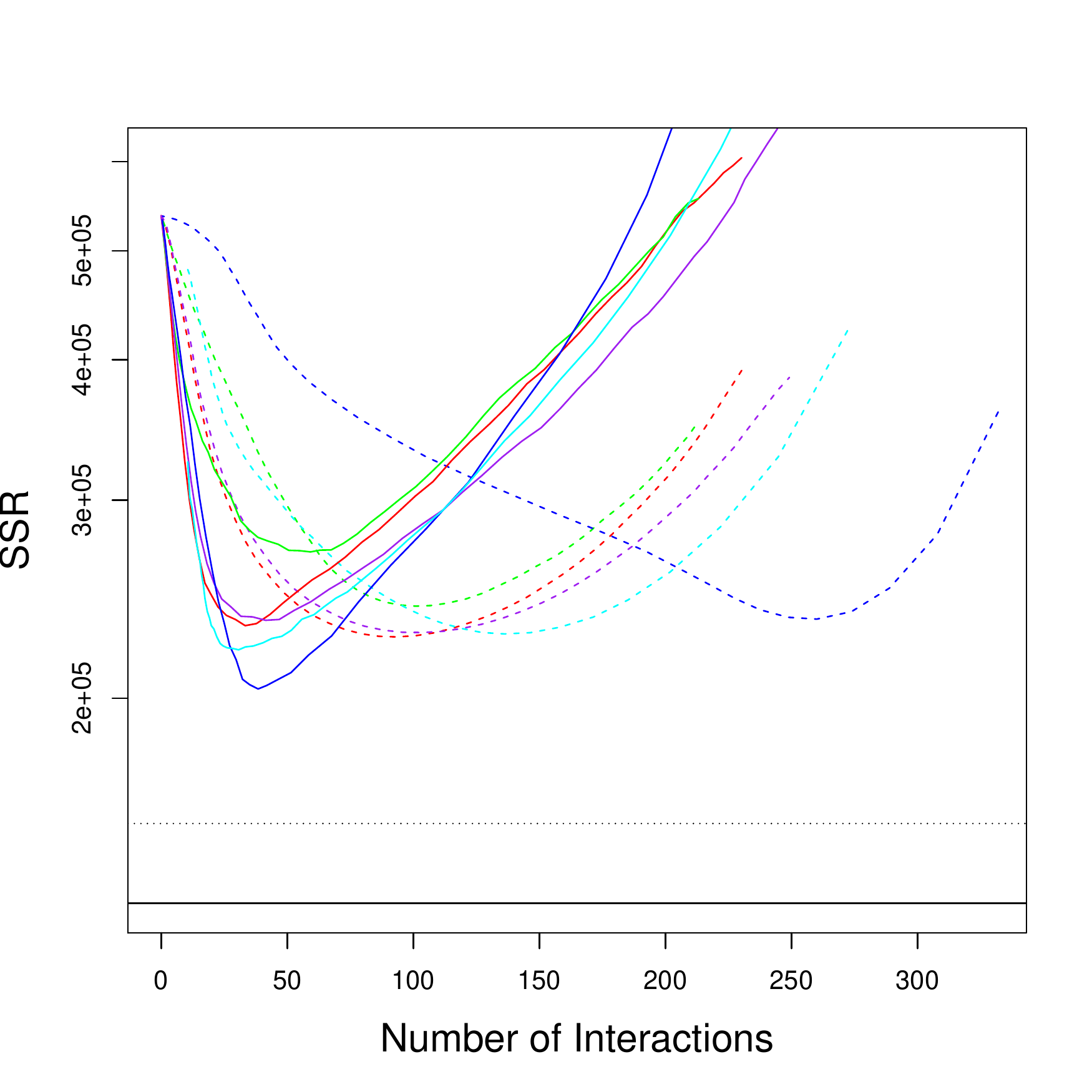} \\
45 Non-Zero Interactions \\ 
\includegraphics[scale = 0.33]{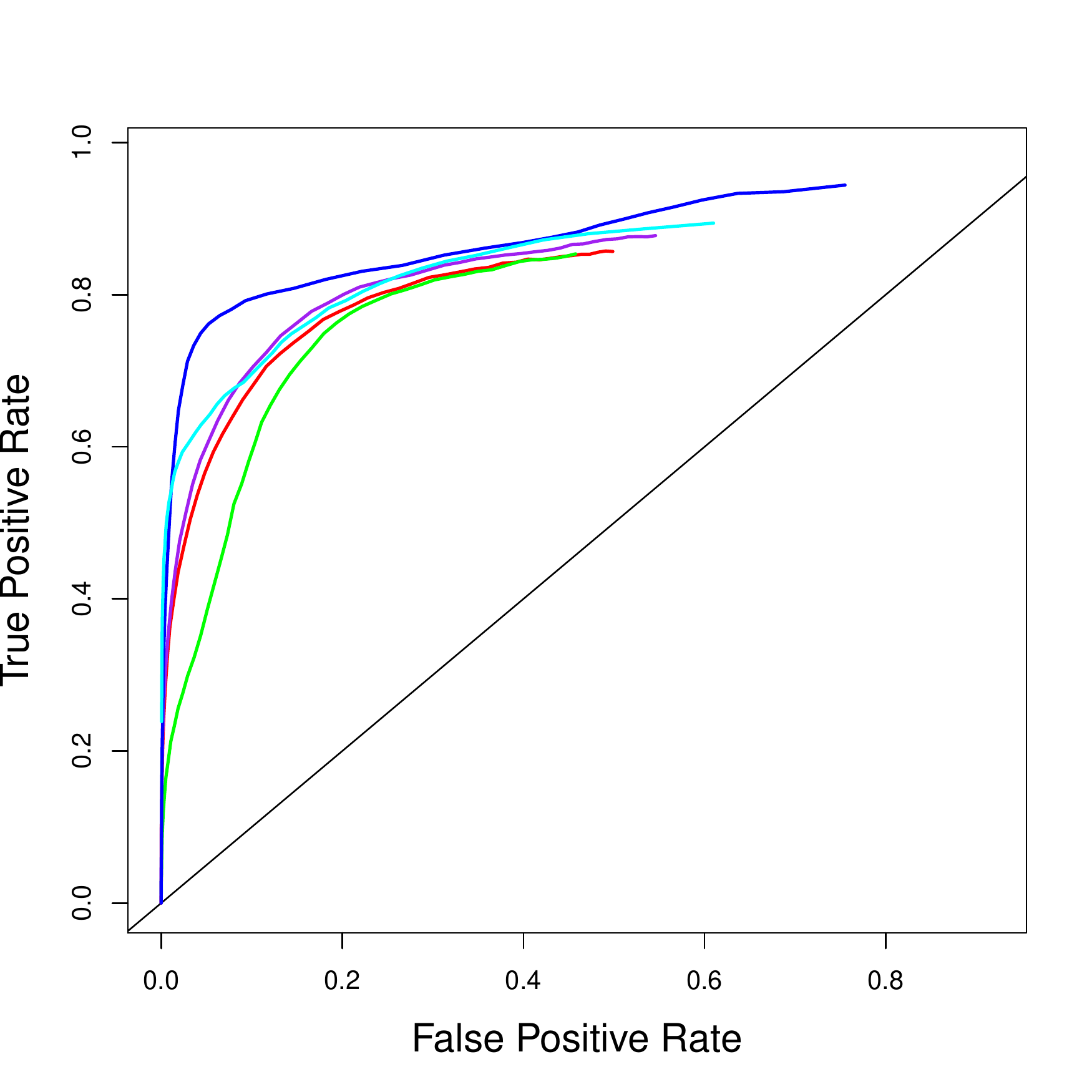} 
\includegraphics[scale = 0.33]{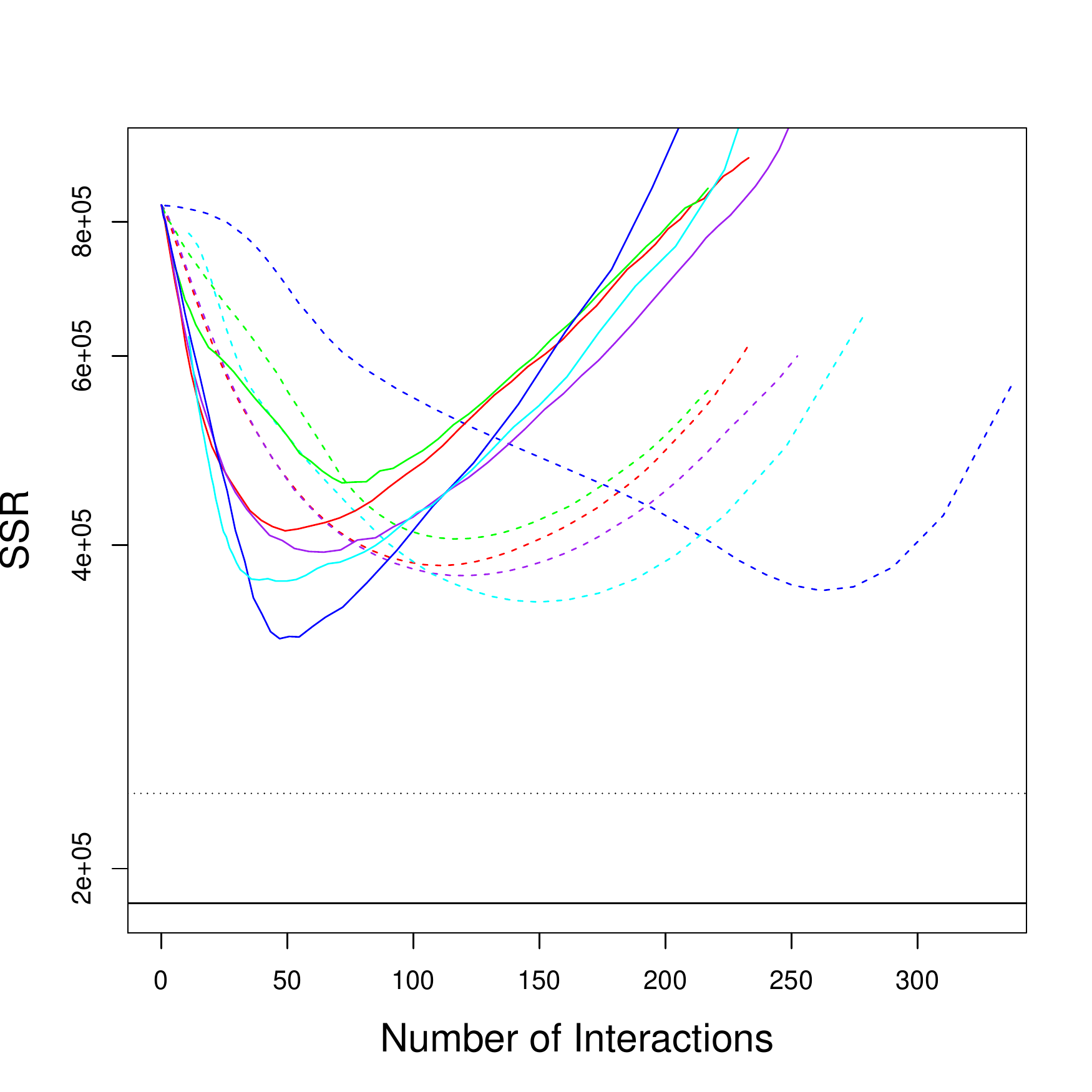} \\
\caption{ Results for the simulation study of Section~\ref{sec:squared}, averaged over 100 simulated datasets.
The colored lines indicate the results for \glinternet\ (\protect\includegraphics[height=0.5em]{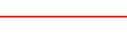}), \hiernet\ (\protect\includegraphics[height=0.5em]{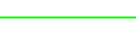}), \APL\ (\protect\includegraphics[height=0.5em]{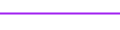}), \familylt\ with $\alpha=0.7$ (\protect\includegraphics[height=0.5em]{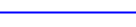}), and \familyli\ with $\alpha=0.83$ (\protect\includegraphics[height=0.5em]{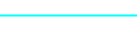}).
\emph{Left:} ROC curves for each proposal, along with the $45 ^{\circ}$ line. \emph{Right:} Sum of squared residuals (SSR), evaluated on the test set. Each method is shown with (\protect\includegraphics[height=0.5em]{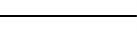}) and without (\protect\includegraphics[height=0.5em]{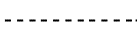}) relaxation. The two horizontal black lines indicate the test set SSR of the true model (\protect\includegraphics[height=0.5em]{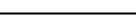}) and of the oracle model (\protect\includegraphics[height=0.5em]{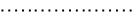}). 
}
\label{fig:result}
\end{figure}

\begin{table}[H]
\begin{center}
\resizebox{12cm}{!}{
\begin{tabular}{|c|c|c|c|c|c|r|}
\hline
\small 
 &{Method } &Relaxed  & Relative SSR&{FDR} &{TPR} &{Num.  Inter.} \\
\hline
\multirow{10}{*}{15}
& \multirow{2}{*}{\familylt} &  No & 1.333 (0.012) & 0.892 (0.002) & 0.931 (0.006) & 132.01 (2.3) \\
&           & Yes & 1.133 (0.010) & 0.399 (0.017) & 0.837 (0.009) & 22.94 (0.8) \\
\cline{2-7}
& \multirow{2}{*}{\familyli} &  No & 1.348 (0.011) & 0.855 (0.003) & 0.915 (0.006) & 97.85 (1.7) \\
&           & Yes & 1.179 (0.011) & 0.304 (0.017) & 0.771 (0.010) & 17.87 (0.6) \\
\cline{2-7}
 & \multirow{2}{*}{\glinternet} &  No & 1.288 (0.011) & 0.786 (0.004) & 0.889 (0.007) & 64.85 (1.4) \\
& 				 & Yes & 1.230 (0.010) & 0.209 (0.017) & 0.691 (0.011) & 14.23 (0.6) \\
\cline{2-7}
& \multirow{2}{*}{\hiernet} &  No & 1.359 (0.012) & 0.816 (0.003) & 0.881 (0.007) & 73.12 (1.2) \\
& 		   & Yes & 1.355 (0.013) & 0.382 (0.023) & 0.632 (0.013) & 19.76 (1.4) \\
\cline{2-7}
& \multirow{2}{*}{\APL} 	&  No & 1.341 (0.011) & 0.816 (0.004) & 0.895 (0.007) & 75.90 (1.6) \\
& 		& Yes & 1.308 (0.012) & 0.375 (0.019) & 0.749 (0.011) & 20.65 (1.0) \\
\hline

\multirow{10}{*}{30}
& \multirow{2}{*}{\familylt} &  No & 1.492 (0.016) & 0.841 (0.003) & 0.884 (0.006) & 172.00 (3.3) \\
&           & Yes & 1.218 (0.012) & 0.352 (0.014) & 0.800 (0.010) & 39.09 (1.1) \\
\cline{2-7}
& \multirow{2}{*}{\familyli} &  No & 1.476 (0.016) & 0.790 (0.004) & 0.846 (0.007) & 124.00 (2.2) \\
&           & Yes & 1.276 (0.013) & 0.310 (0.016) & 0.735 (0.008) & 34.11 (1.0) \\
\cline{2-7}
 & \multirow{2}{*}{\glinternet} &  No & 1.487 (0.015) & 0.730 (0.005) & 0.800 (0.007) & 91.75 (1.8) \\
& 				 & Yes & 1.446 (0.016) & 0.328 (0.017) & 0.627 (0.010) & 31.07 (1.3) \\  
\cline{2-7}
& \multirow{2}{*}{\hiernet} &  No & 1.567 (0.016) & 0.754 (0.003) & 0.797 (0.008) & 98.95 (1.7) \\
& 		   & Yes & 1.677 (0.019) & 0.581 (0.013) & 0.647 (0.012) & 50.90 (1.8) \\
\cline{2-7}
& \multirow{2}{*}{\APL} &  No & 1.492 (0.016) & 0.751 (0.004) & 0.821 (0.007) & 101.73 (1.8) \\
& 	  & Yes & 1.484 (0.018) & 0.411 (0.016) & 0.676 (0.010) & 37.78 (1.4)\\
\hline

\multirow{10}{*}{45}
& \multirow{2}{*}{\familylt} &  No & 1.562 (0.020) & 0.816 (0.003) & 0.889 (0.005) & 223.29 (4.0) \\
&           & Yes & 1.219 (0.016) & 0.203 (0.016) & 0.833 (0.008) & 49.09 (1.2) \\
\cline{2-7}
& \multirow{2}{*}{\familyli} &  No & 1.531 (0.019) & 0.754 (0.003) & 0.841 (0.006) & 156.59 (2.6) \\
&           & Yes & 1.324 (0.023) & 0.200 (0.019) & 0.756 (0.009) & 45.78 (1.5) \\
\cline{2-7}
 & \multirow{2}{*}{\glinternet} &  No & 1.658 (0.021) & 0.679 (0.004) & 0.776 (0.005) & 110.28 (1.4) \\
& 				 & Yes & 1.689 (0.025) & 0.415 (0.012) & 0.610 (0.009) & 50.07 (1.7) \\
\cline{2-7}
& \multirow{2}{*}{\hiernet} & No  & 1.746 (0.023) & 0.699 (0.003) & 0.772 (0.006) & 116.46 (1.5) \\
& 		   & Yes & 1.876 (0.027) & 0.585 (0.006) & 0.650 (0.008) & 72.29 (1.5) \\
\cline{2-7}
& \multirow{2}{*}{\APL} &  No & 1.616 (0.021) & 0.693 (0.004) & 0.802 (0.005) & 119.73 (1.8) \\
& 	  & Yes & 1.633 (0.023) & 0.456 (0.012) & 0.674 (0.008) & 59.40 (1.8) \\

\hline

\end{tabular}
}
\caption{Simulation results, averaged over 100 simulated datasets, for the simulation set-up in Section~\ref{sec:squared}. Tuning parameters were selected using a training/test/validation set approach, as described in Section \ref{sec:datagen}. From left to right, the table's columns indicate the true number of non-zero interactions, the method used, whether or not relaxation was performed, the sum of squared residuals (SSR) on the validation set divided by the SSR of the oracle, the false discovery rate for the detection of non-zero interactions, the true positive rate for the detection of non-zero interactions, and the number of estimated non-zero interactions. Standard errors of the mean are reported in parentheses. }
\label{tab1}
\end{center}
\end{table}

%The forward selection proposal of \citet{hao2014interaction}, \iFORM, is a fast screening tool targeting ultra-high dimensional data. Another advantage is \iFORM's good variable selection properties for extremely sparse models. In our simulation setting, we consider moderately sparse models, which fails to highlight the main advantages of \iFORM. Thus, we do not include \iFORM in our simulation study.

%The forward selection proposal of \citet{hao2014interaction}, \iFORM, is not optimal for our study. This is because \iFORM\ requires the coefficients for the main effects to be sufficiently large that they can be  picked up without considering interactions. In contrast, in our simulation study, the main effects are moderate and can be picked up only if selected together with interactions.  
%One of \iFORM's most attractive properties is its speed relative to regularized regression approaches in the ultra-high-dimensional setting; however, our simulation set-up is moderate-dimensional, and so does not highlight this advantage.  The results for \iFORM\ are omitted from Table~\ref{tab1} and Figure~\ref{fig:result}.

\subsection{Logistic regression}
\label{sec:SimStudyGlm}

\subsubsection{Simulation Set-up}
\label{sec:datagenGlm}
We assume that each response $y_i$ is a Bernoulli variable with probability $p_i$. We then model $p_i$ as
\begin{equation}
\log\left( \frac{p_i}{1-p_i}\right)  = (W*B)_i;  \ i = 1,\ldots , n, 
\label{modelGLM}
\end{equation} 
where $W*B$ is the $n$-vector defined in Section \ref{sec:overallModel}. 
 The matrices $X$ and $B$ are generated in the exact same manner as in Section \ref{sec:datagen}, but now with $n=500$ observations in the training and test sets.

Once again, for convenience, we reparametrized \familylt \ and \familyli \ according to
\begin{equation}
\begin{split}
\underset{B\in \mathbb{R}^{(p+1)\times (p+1)}}{\text{minimize}}  &-\frac{1}{n}\sum_{i=1}^{n}\left[ y_i(W*B)_i - \log\left( 1+ e^{(W*B)_i} \right) \right] \\
&+  \sqrt{p}(1-\alpha)\lambda \sum_{j=1}^{p}P_r( B_{j,.})
+ \sqrt{p}(1-\alpha)\lambda \sum_{k=1}^{p}P_c (B_{.,k})+ \alpha\lambda \|B_{-0,-0}\|_1.
\end{split}
\label{eqn:objglm-reparametrized}
\end{equation}

\subsubsection{Results}
\label{sec:resultsGlm}
The results for logistic regression are displayed in Figure~\ref{fig:glmresult}. The ROC curves in the left-hand panel indicate that \familyli \ and \familylt \ outperform the competitors in terms of variable selection when there are 30 or 45 non-zero interactions. The SSR curves in the right-hand panel of Figure~\ref{fig:glmresult} indicate that the relaxed versions of \familyli\ and \familylt\ perform very well in terms of prediction error on the test  set, especially as the number of non-zero interactions increases.

\begin{figure}
\centering
15 Non-Zero Interactions \\ 
\includegraphics[scale = 0.33]{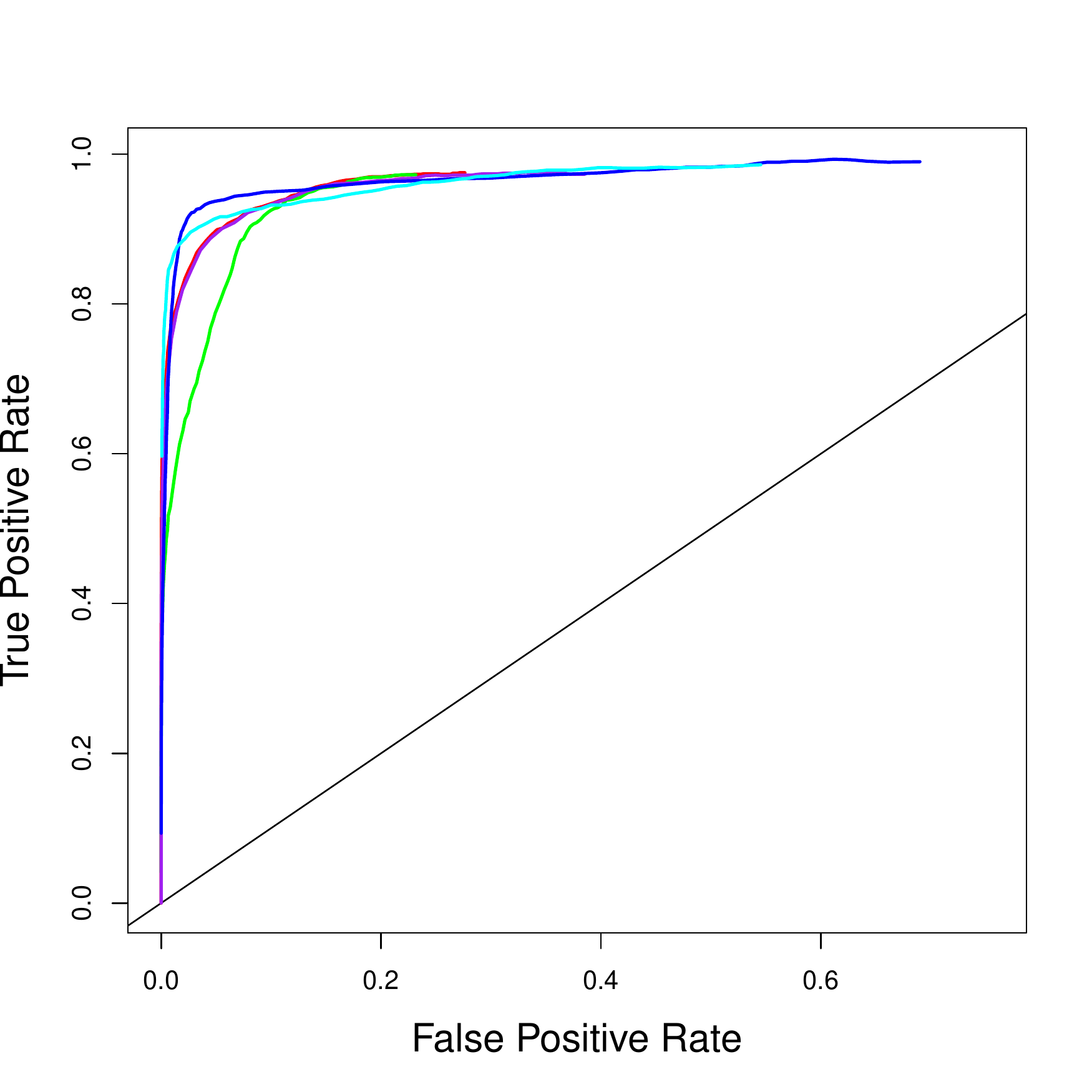} 
\includegraphics[scale = 0.33]{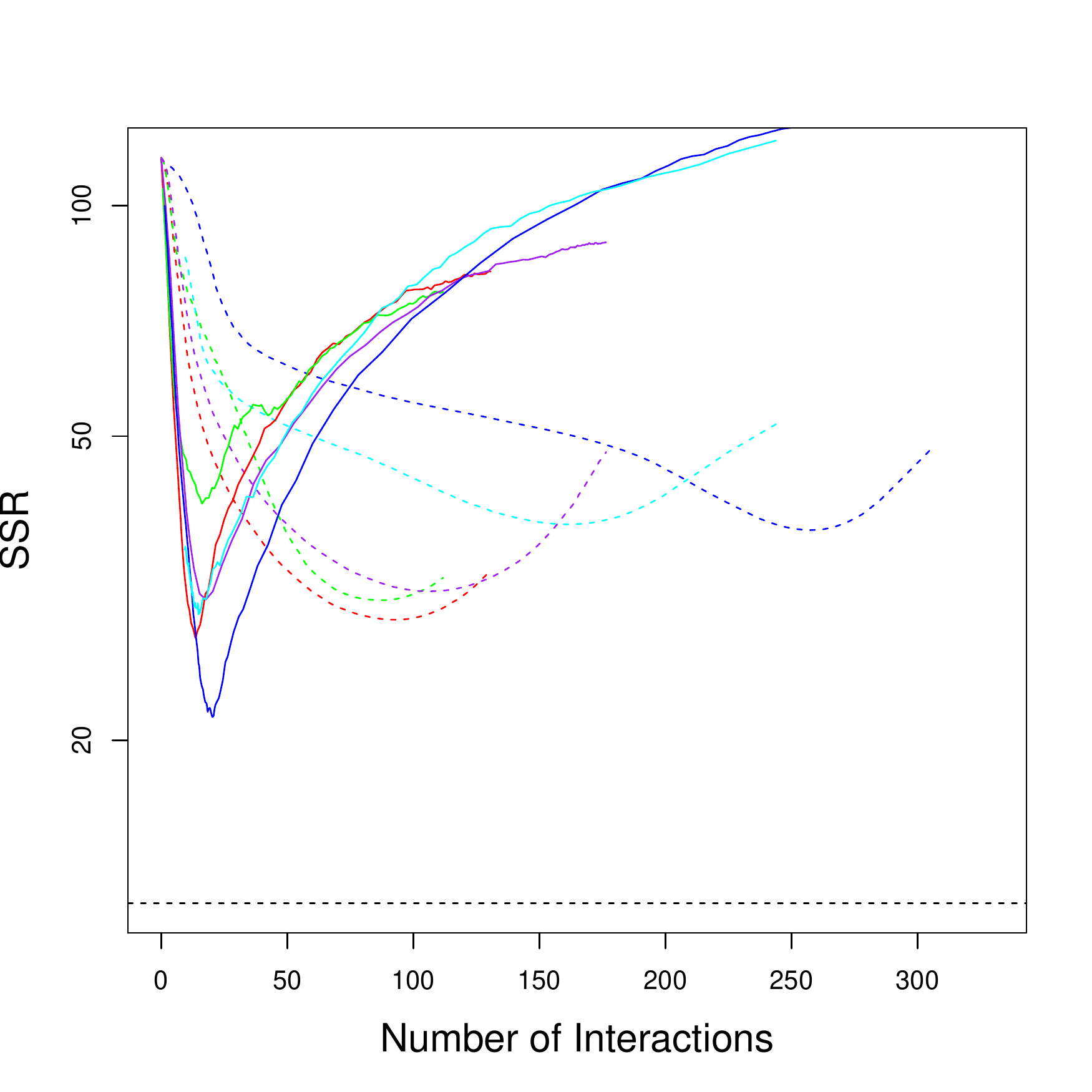} \\ 
30 Non-Zero Interactions \\ 
\includegraphics[scale = 0.33]{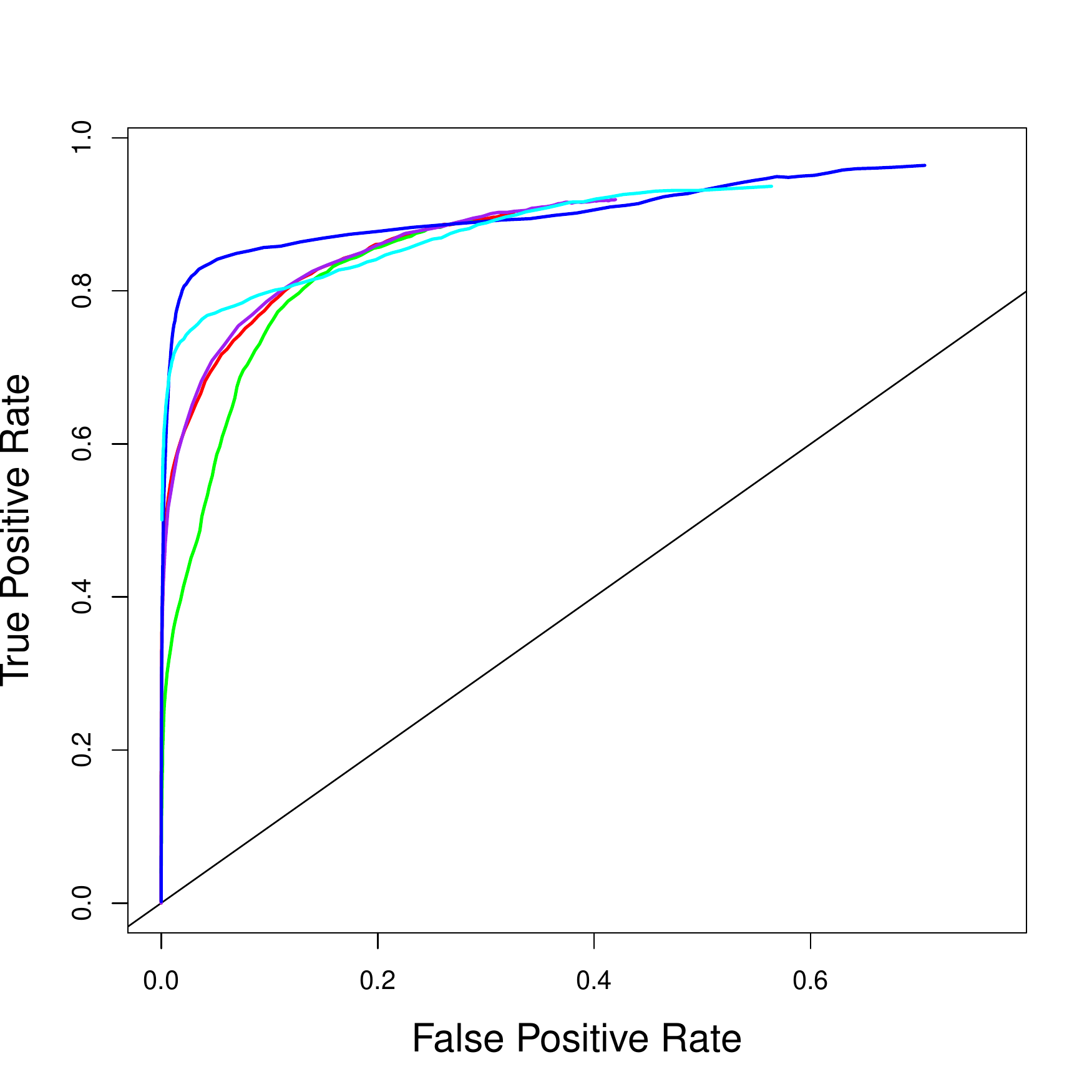} 
\includegraphics[scale = 0.33]{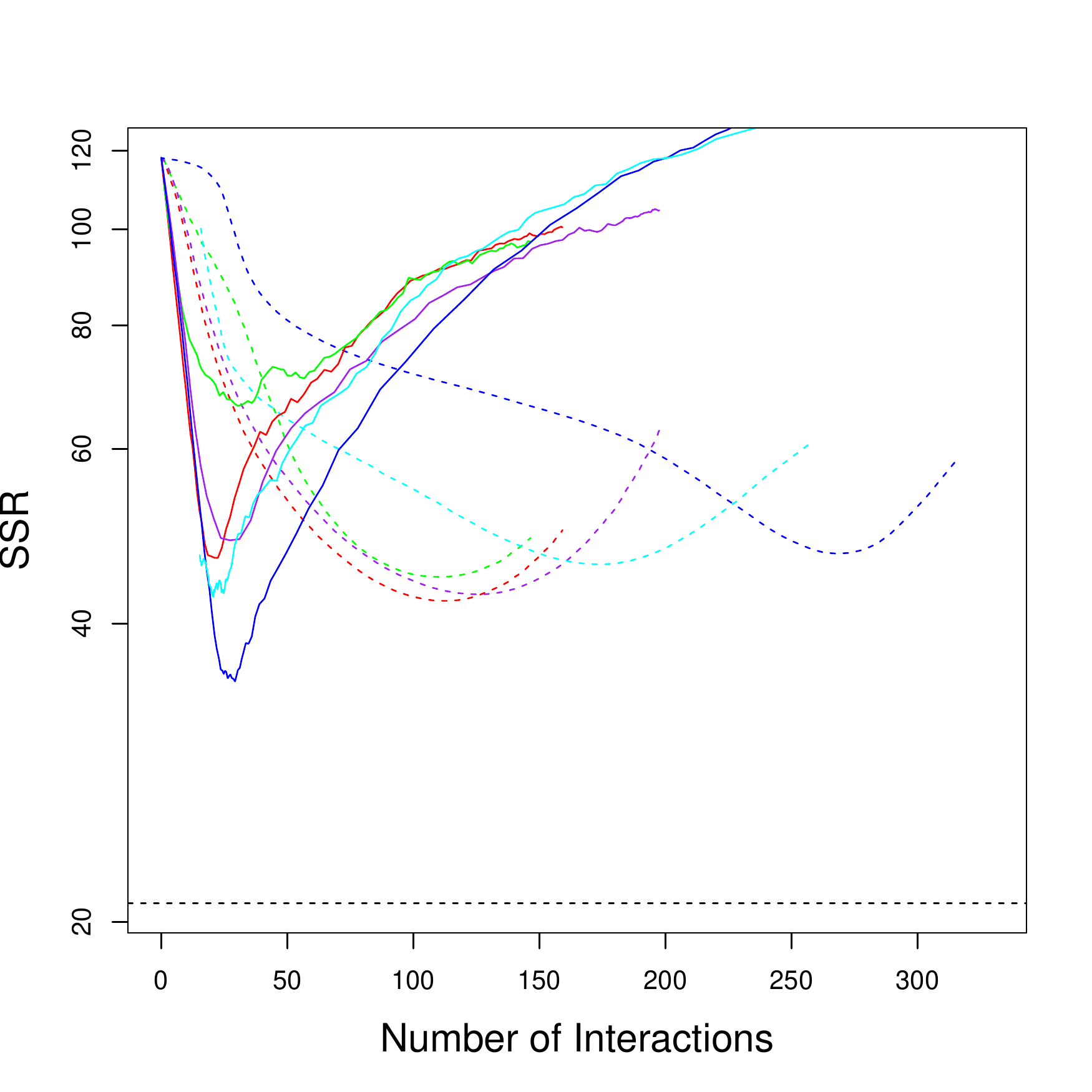} \\
45 Non-Zero Interactions \\ 
\includegraphics[scale = 0.33]{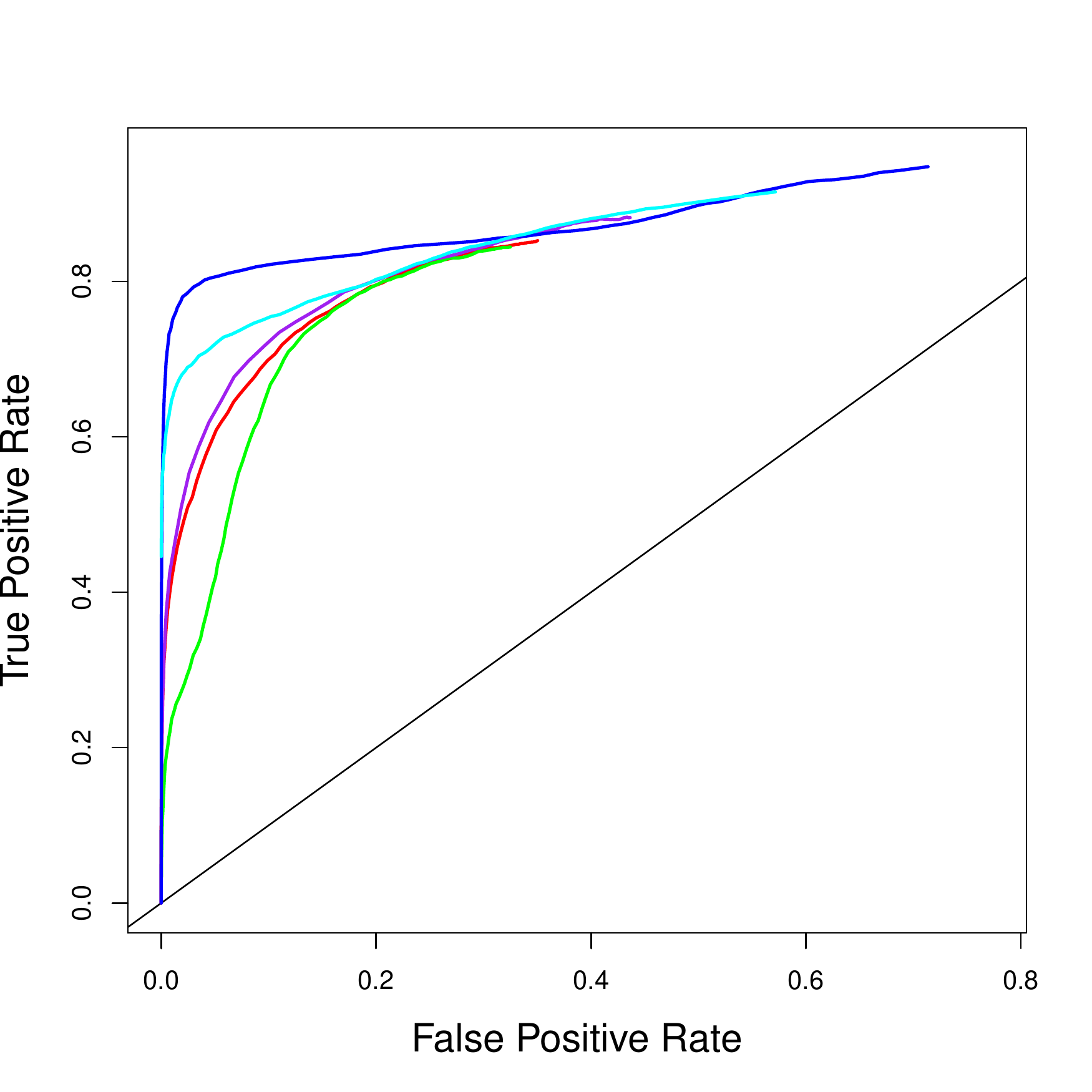} 
\includegraphics[scale = 0.33]{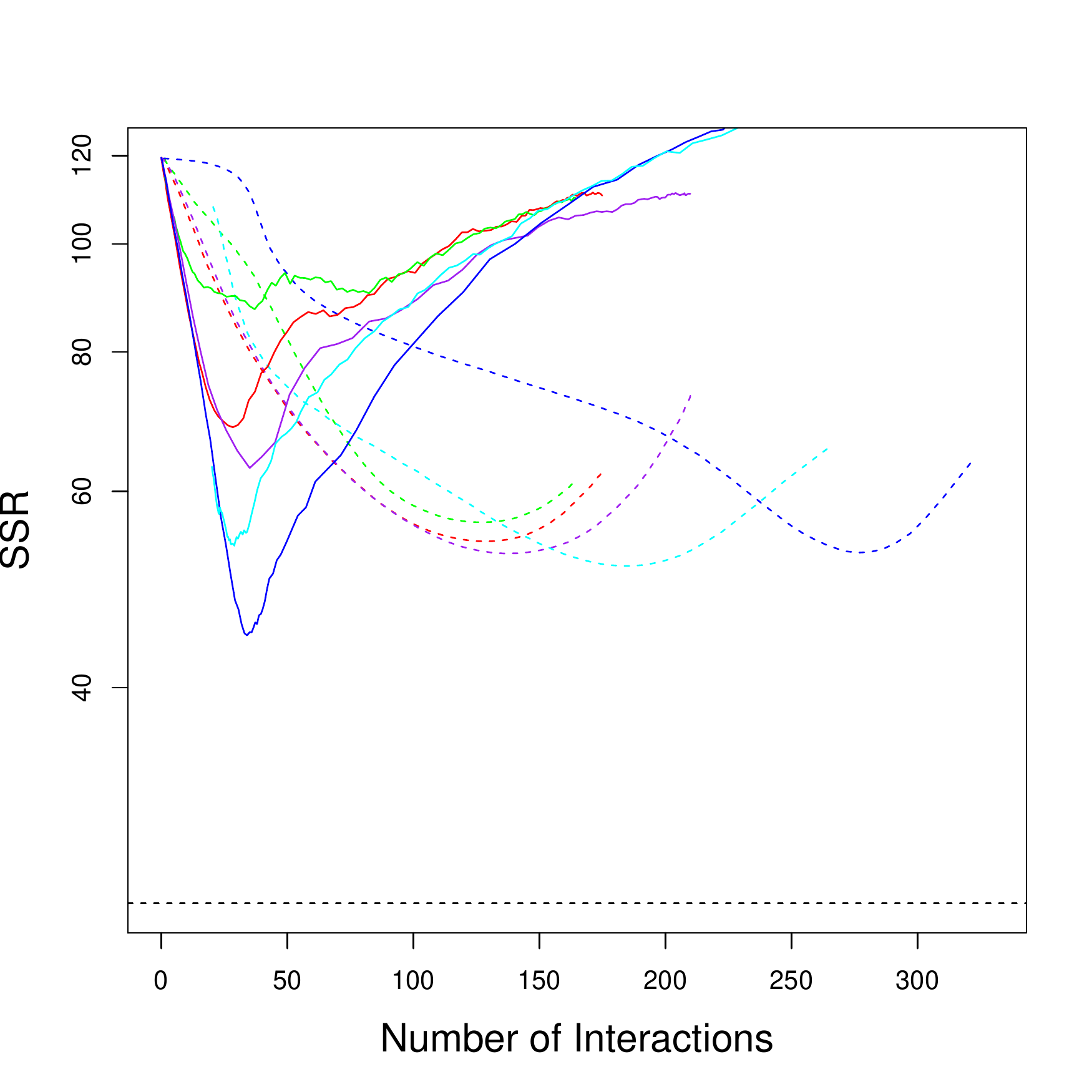} 
\caption{ Results for the simulation study of Section~\ref{sec:SimStudyGlm}, averaged over 100 simulated data sets. Details are as in Figure~\ref{fig:result}, but with $\alpha=0.8$ for \familyli\ (\protect\includegraphics[height=0.5em]{res-fami.png}). }
\label{fig:glmresult}
\end{figure}

\section{Application to HIV Data}
\label{sec:realData}

\citet{rhee2006genotypic} study the susceptibility of the HIV-1 virus to 6  nucleoside reverse transcriptase inhibitors (NRTIs).  The HIV-1 virus can become resistant to drugs via mutations in its genome sequence. Therefore, there is a need to model HIV-1's drug susceptibility  as a function of mutation status. We consider one particular NRTI, 3TC. The data consists of a sparse binary matrix, with  mutation status at each of $217$ genomic locations for $n=1057$ HIV-1 isolates.  For each of the observations, there is a measure of susceptibility to 3TC. This data set was also studied by \cite{bien2013lasso}.

Rather than working with all $217$ genomic locations, we create bins of ten adjacent loci; this results in a design matrix with $p=22$ features and $n=1057$ observations. We perform the binning because the raw data contains mostly zeros, as most mutations occur in at most a few of the observations; by binning the observations, we obtain less sparse data. This binning is justified under the assumption that mutations in a particular region of the genome sequence result in a change to a binding site, in which case nearby mutations should have similar effects on a binding site, and hence similar associations with drug susceptibility. This binning is also needed for computational reasons, in order to allow for comparison to \hiernet\ (specifically the version that enforces strong heredity) using the \verb=R= package of \citet{bien2013lasso}.~(In \citet{bien2013lasso}, all $217$ genomic locations are analyzed using a much faster algorithm that enforces \emph{weak} (rather than strong) heredity.)

We split the observations into equally-sized training and test sets. We fit \glinternet, \hiernet, \familylt, and \familyli\ on 
the training set for a range of tuning parameter values, and applied the fitted models to the test set.
In Figure \ref{fig:dataSSR}, the test set SSR is displayed as a function  of the number of non-zero estimated interaction coefficients,  averaged over 50 splits of the data into training and test sets. 
The figure reveals that all four methods give roughly similar results.%, though \familylt \ seems to perform slightly better than the competitors for models containing a greater number of non-zero interaction terms. 

\begin{figure}[H]
\centering
\includegraphics[scale = 0.4]{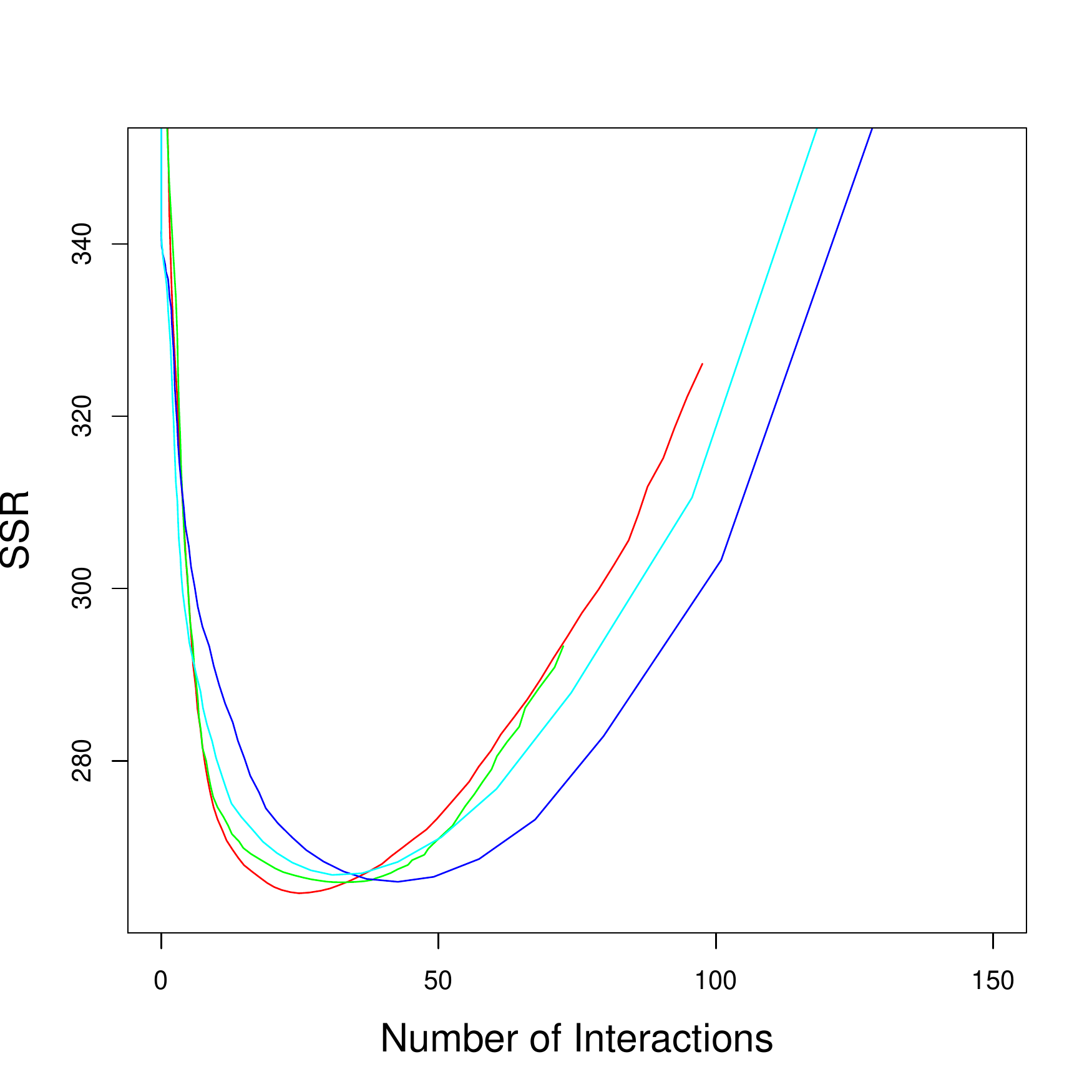}   
\caption{ The test set SSR is displayed for the HIV-1 data of Section~\ref{sec:realData}, as a function of the number of non-zero interaction terms. Results are averaged over 50 splits of the observations into a training set and a test set. The colored lines indicate the results for \glinternet\ (\protect\includegraphics[height=0.5em]{res-glin.png}) , \hiernet\ (\protect\includegraphics[height=0.5em]{res-hier.png}), \familylt\ with $\alpha=0.944$ (\protect\includegraphics[height=0.5em]{res-faml2.png}), and \familyli\ with $\alpha=0.944$ (\protect\includegraphics[height=0.5em]{res-fami.png}).
 }
\label{fig:dataSSR}
\end{figure}

Figure~\ref{fig:data} displays the estimated coefficient matrix, $\hat{B}$, that results from applying each of the four methods to all $n=1057$ observations using the tuning parameter values that minimized the average test set SSR. The estimated coefficients are qualitatively similar for all four methods. 
All four methods detect some non-zero interactions involving the 17th feature. \texttt{Glinternet} yields the sparsest model.

\begin{figure}
\centering
(a) \hspace{60mm} (b) \\
\includegraphics[scale = 0.32]{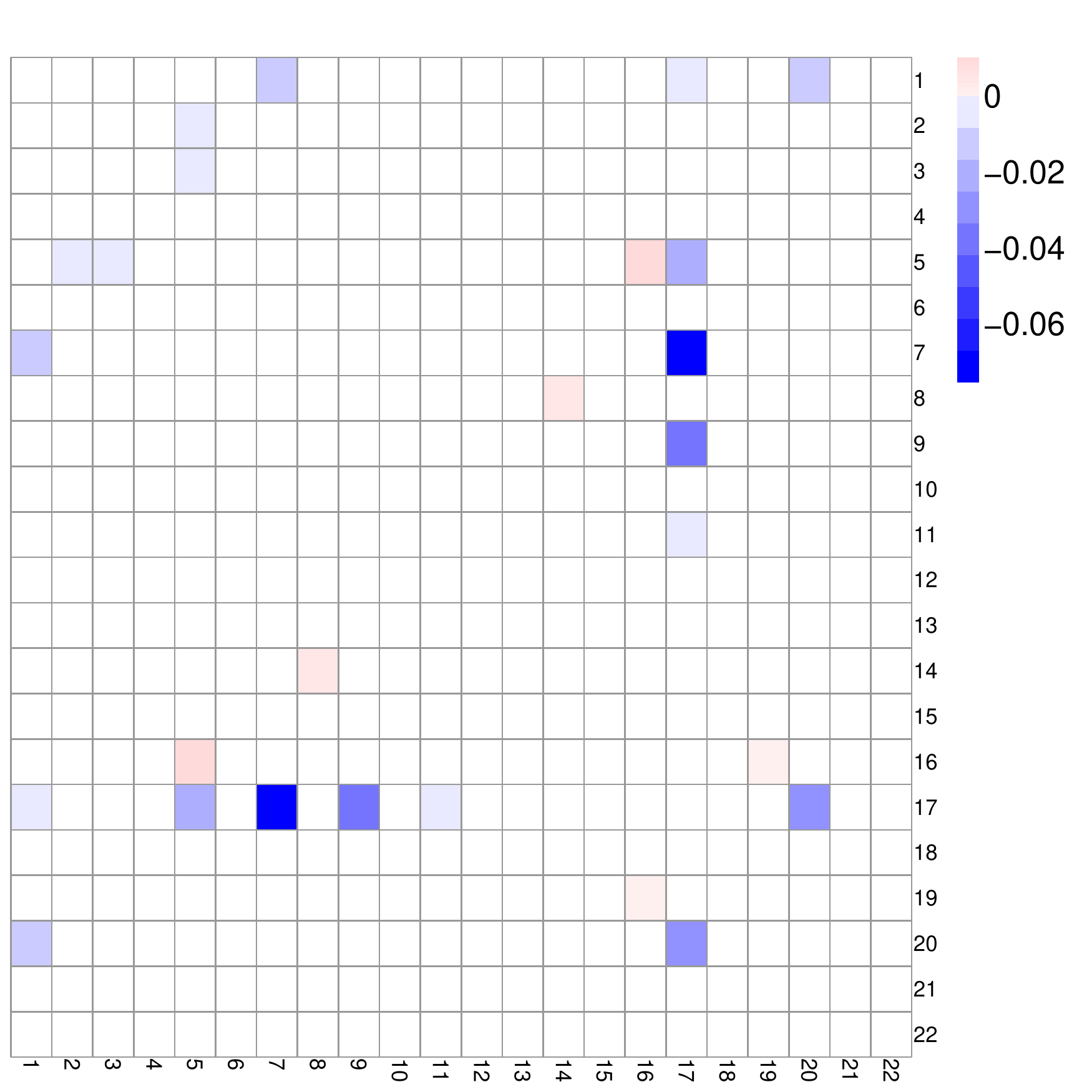}
\includegraphics[scale = 0.32]{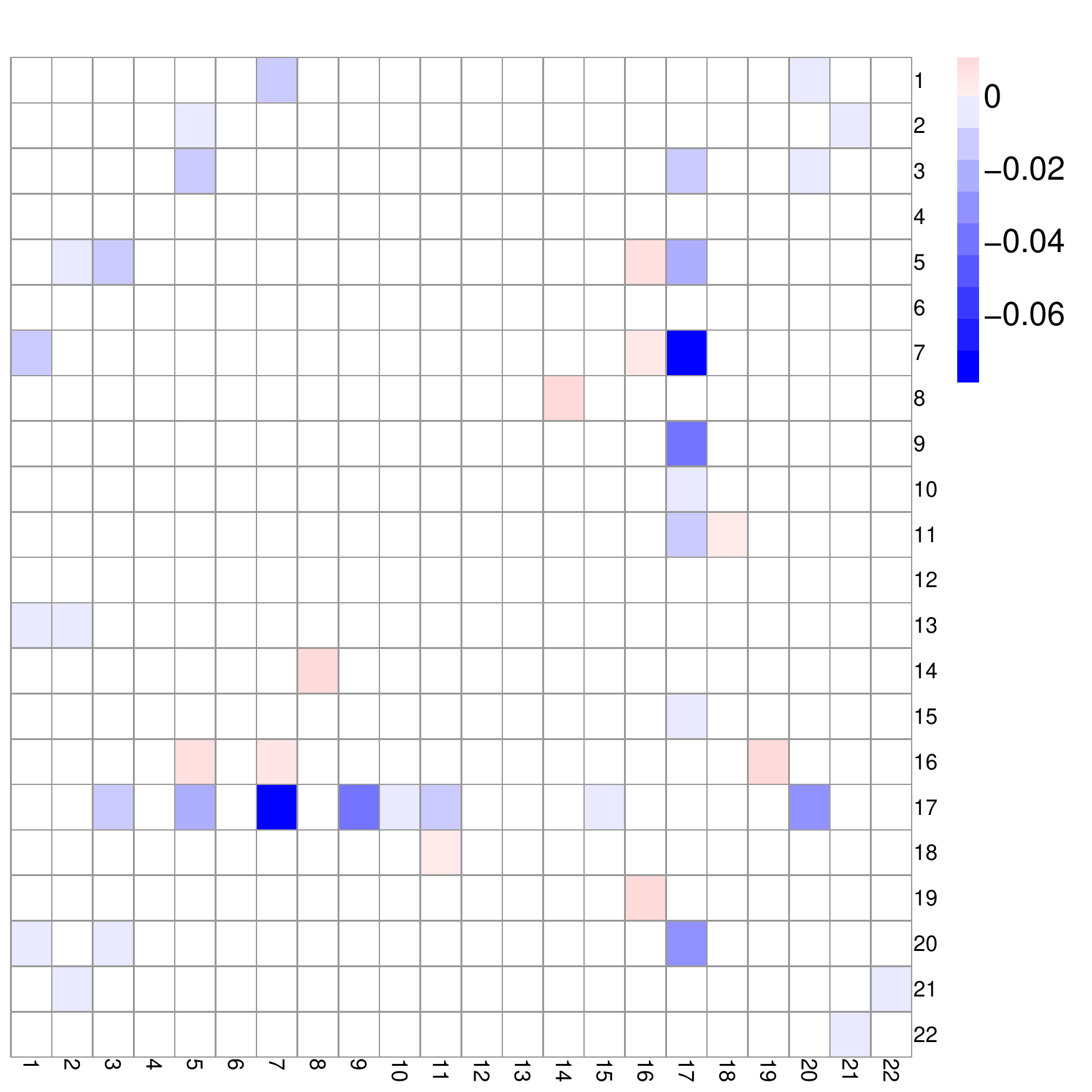}\\
(c)  \hspace{60mm} (d) \\
\includegraphics[scale = 0.32]{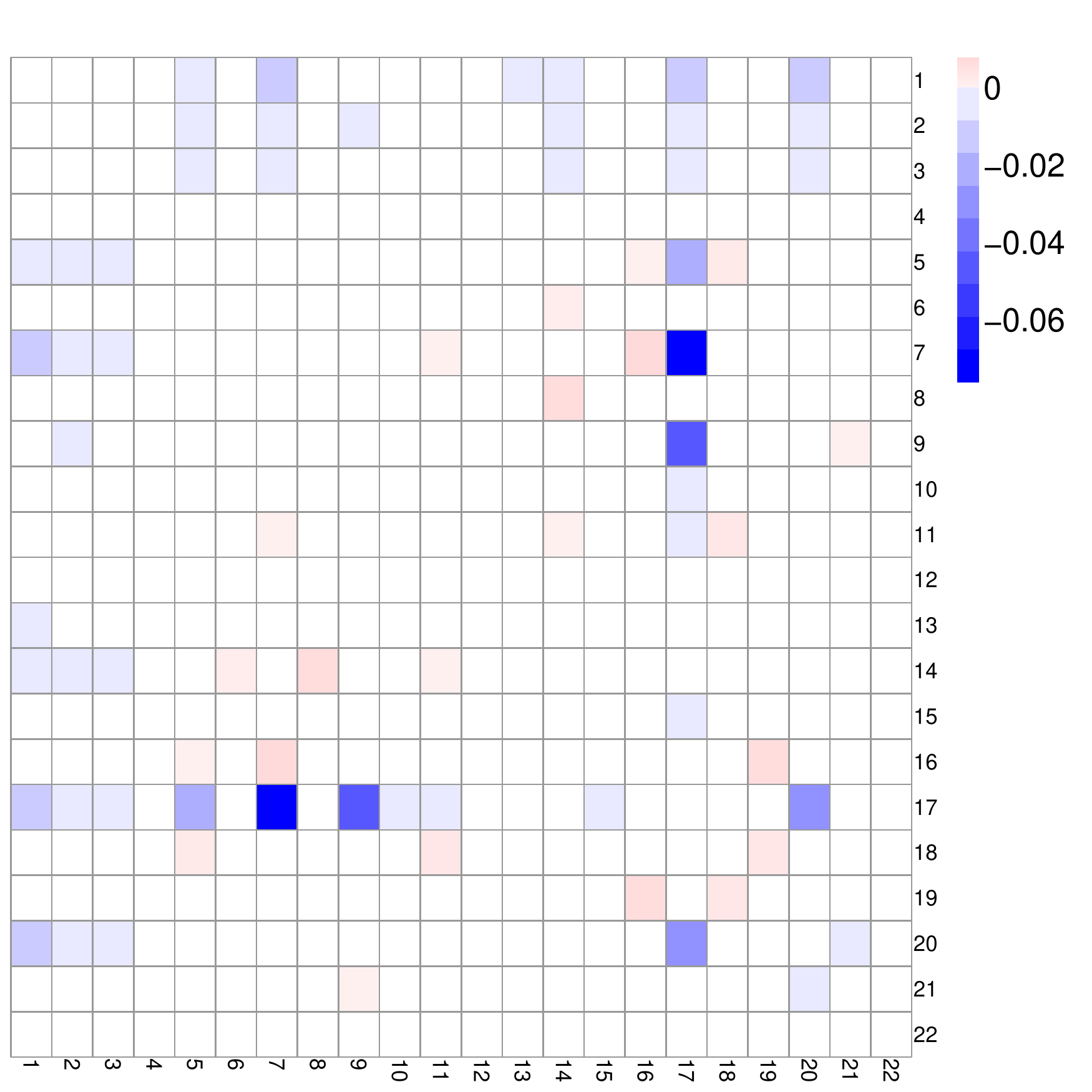}
\includegraphics[scale = 0.32]{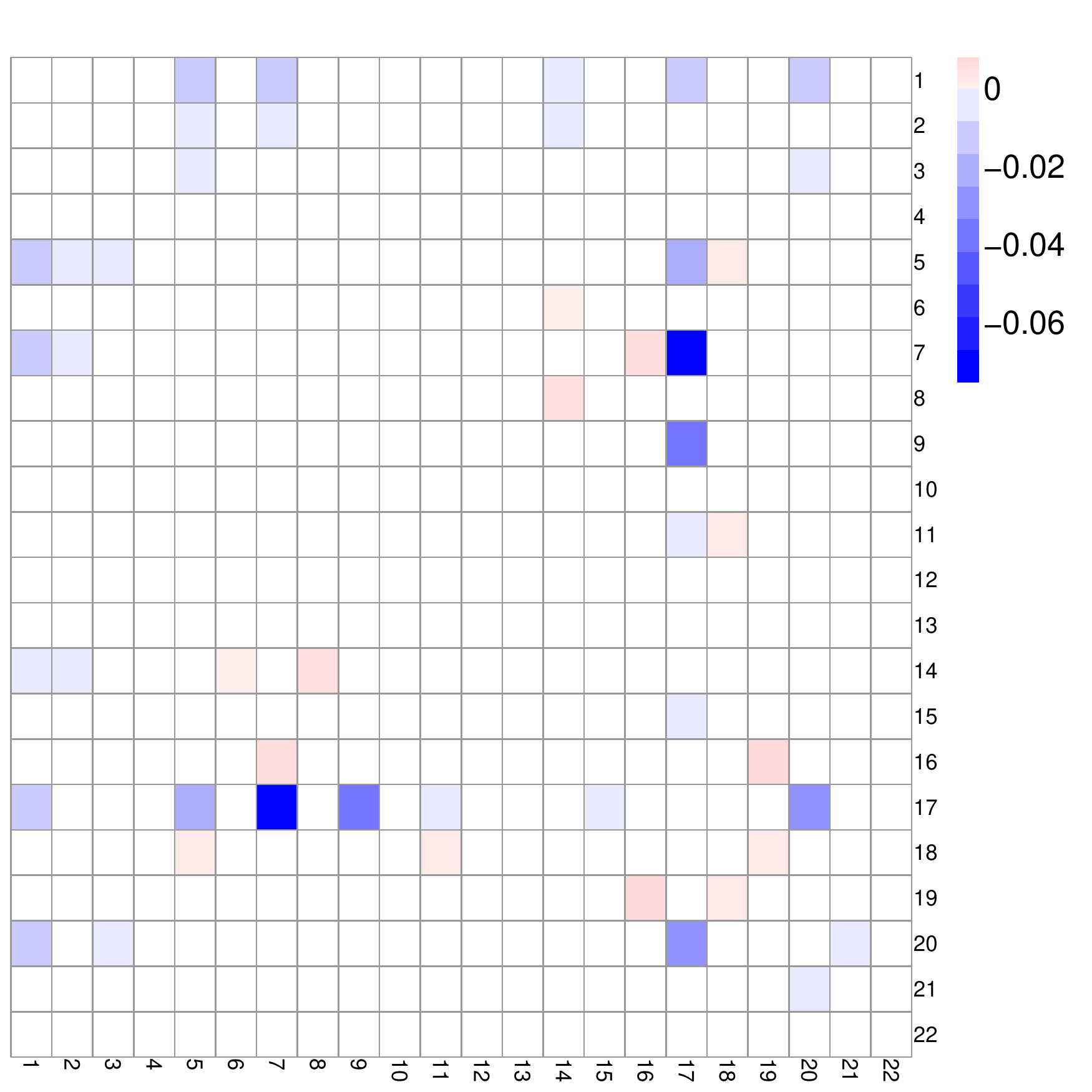}
\caption{For the HIV-1 data of Section~\ref{sec:realData}, the estimated coefficient matrix $\hat{B}_{-0,-0}$ is shown for \emph{(a):} \glinternet; \emph{(b):} \hiernet; \emph{(c):} \familylt\ with $\alpha=0.944$; and \emph{(d):} \familyli\ with $\alpha=0.944$. Main effects are not displayed. }
 \label{fig:data}
\end{figure}

\section{Conclusion}
\label{sec:conclusion}

In this paper, we have  introduced  \family, a framework that unifies a number of existing estimators for high-dimensional models with interactions. 
Special cases of \family\ correspond to the all-pairs lasso, the main effects lasso, \vanish, and \hiernet. Furthermore,  we have explored the use of \family\ with $\ell_2$, $\ell_\infty$, and hybrid $\ell_1$/$\ell_\infty$ penalties; these result in strong heredity and have good empirical performance.  

The empirical results in Sections~\ref{sec:SimStudy} and \ref{sec:realData} indicate that the choice of penalty in \family\ may be of little practical importance: for instance, \familylt, \familyli, and \familyhiernet\ have similar performance. However, one could choose among penalties  using cross-validation or a related approach.

We have presented a simple ADMM algorithm that can be used to solve the \family\ optimization problem for any convex penalty. It finds the global optimum for \vanish\ (unlike the  proposal in \cite{radchenko2010variable}), and provides a simpler alternative to the original \hiernet\ algorithm  \citep{bien2013lasso}.

 \family\ could be easily extended to accommodate higher-order interaction models. For instance, to accommodate third-order interactions, we could take $B$ to be a $(p+1)\times (p+1) \times (p+1)$ coefficient array. Instead of penalizing each  row and each column of $B$, we would instead penalize each `slice' of the array. 
 
 In the simulation study in Section~\ref{sec:SimStudy}, we considered a setting with only $p_1=p_2=30$ main effects. We did this in order to facilitate comparison to the \hiernet\ proposal, which is very computationally intensive as implemented in the \verb=R= package of \cite{bien2013lasso}. However, our proposal can be applied for much larger values of $p_1$ and $p_2$, as discussed in Section~\ref{sec:computationAlg}.

The \verb=R= package \verb=FAMILY=, available on \verb=CRAN=, implements the methods described in this paper. 

\section*{Acknowledgments}
We thank an anonymous associate editor and two referees for insightful comments that resulted in substantial improvements to this manuscript. We thank Helen Hao Zhang, Ning Hao, Jacob Bien, Michael Lim, and Trevor Hastie for providing software and helpful responses to inquiries. D.W. was supported by NIH Grant DP5OD009145, NSF CAREER Award DMS-1252624, and an Alfred P. Sloan Foundation Research Fellowship. N.S. was supported by NIH Grant DP5OD019820.

\singlespacing

%\newpage

\bibliographystyle{plainnat}
\bibliography{references/ref}

%\newpage 

\appendix
\numberwithin{equation}{section}

\section{Alternating Directions Method of Multipliers}

\subsection{Overview of ADMM}

%\subsection{Alternating Directions Method of Multipliers}
\label{sec:admm}
We will solve \eqref{obj} using  the \emph{alternating directions method of multipliers} (ADMM) algorithm, which we briefly review here. We refer the reader to \citet{boyd2011distributed} for a detailed discussion.

ADMM provides a simple, general, and efficient approach for solving a problem of the form
\begin{equation}
\underset{x}{\text{minimize }}f_1(x)+f_2(x),
\label{admm}
\end{equation} 
%\begin{flushleft}
where $f_1$ and $f_2$ are convex, closed and proper. The key insight behind ADMM is that (\ref{admm}) can be re-written as 
%\end{flushleft} 
\begin{equation}
%\begin{split}
\underset{x,y}{\text{minimize }} \{ f_1(x)+f_2(y)  \} 
\mbox{ subject to } x=y.
%\end{split}
\label{admm2}
\end{equation}
The augmented Lagrangian corresponding to (\ref{admm2}) takes the form
\begin{equation*}
L_{\rho}(x,y,\gamma) = f_1(x)+f_2(y) + \gamma(x-y) + (\rho/2)\|x-y\|_2^2,
\end{equation*} 
where $\gamma$ is a dual variable and $\rho\in \mathbb{R} $ is a positive constant. The resulting ADMM algorithm involves iterating the following steps until convergence,
\begin{eqnarray*}
x^{k+1} &= \underset{x}{\text{argmin }} L_{\rho} (x,y^k,\gamma^k)\\
y^{k+1} &=  \underset{y}{\text{argmin }} L_{\rho} (x^{k+1},y,\gamma^k)\\
\gamma^{k+1} &= \gamma^k +\rho(x^{k+1}-y^{k+1})\ ,
\end{eqnarray*}
where $k$ indexes the iterations. Under a few simple conditions, the ADMM algorithm converges to the global optimum \citep{boyd2011distributed}.

%\subsection{Algorithm for Solving \family  } % with Squared Error Loss}

\subsection{\family\ with Squared Error Loss}

\label{sec:fullAlg}

\subsubsection{The ADMM Algorithm}
The augmented Lagrangian corresponding to \eqref{obj} was given in \eqref{auglag2}. 
The complete ADMM algorithm is as follows:
%The complete ADMM algorithm for solving (\ref{obj}) is as follows:  
\begin{enumerate}
\item Initialize ${\rho^0,\ B^0,\ \Theta^0 \text{ and }\Gamma^0}$.
\item Choose $\varepsilon^{pri}>0$, $\varepsilon^{dual}>0$.
\item Repeat for $i = 1,2,3,... $ until $r^i < \varepsilon^{pri}$ and  $s^i <\varepsilon^{dual}$, where $r^i$ and $s^i$ are the primal and dual residuals, respectively, defined as 
\begin{align*} 
s^i &= \rho^i \|(D^i|E^i|F^i)-(D^{i-1}|E^{i-1}|F^{i-1}) \|_F   \\
r^i &= \|(B^i|B^i|B^i)-(D^i|E^i|F^i) \|_F. 
\end{align*}

\begin{enumerate}
\item Update $\rho^i$ as described in \citet{boyd2011distributed}:
\begin{equation*}
\rho^{i}  = \left\{ \begin{array}{cc}
2\rho^{i-1} & \mbox{if}\;\; r^{i-1}> 10s^{i-1}\\
\rho^{i-1}/2 &  \mbox{if}\;\;  10r^{i-1} < s^{i-1}\\
\rho^{i-1} & \text{ otherwise }
\end{array} \right.. 
\end{equation*} 

\item  Update ${B^i}$ as the solution to the least squares problem:
\begin{equation*}
\begin{split}
B^i = \ &\underset{B}{\text{argmin}} \ \  \frac{1}{2n} \|{y}- {W}*B\|_2^2\\
&+ \frac{3\rho^i}{2} \left\| \frac{1}{3\rho^i} \left[\rho^i(D^{i-1}+E^{i-1}+F^{i-1}) -(\Gamma_1^{i-1}+\Gamma_2^{i-1}+\Gamma_3^{i-1})\right] - B  \right\|_F^2.
\end{split}
\end{equation*}

\item Update ${D^i}$ and $E^i$ using the proximal operators discussed in Section~\ref{sec:prox-all}: 
\begin{eqnarray*}
D^i &= \underset{D}{\text{argmin}} \ \  \frac{\rho^i}{2} \left\|{D}-\left({B^i}+\frac{{\Gamma_1^{i-1}}}{\rho^{i}}\right) \right\|_F^2 + \lambda_1 \sum_{j=1}^{p_1} P_r({D}_{j,.}),\\
E^i &= \underset{E}{\text{argmin}} \ \  \frac{\rho^i}{2} \left\|{E}-\left({B^{i}}+\frac{{\Gamma^{i-1}_2}}{\rho^i}\right) \right\|_F^2 + \lambda_2 \sum_{j=1}^{p_2} P_c({E}_{.,k}) 
\end{eqnarray*}

\item Update $F^i$ as follows: 

\begin{align*}
{F}^{i}_{0,.} &=  {B}^{i}_{0,.}+\frac{{\Gamma_3}^{i-1}_{0,.}}{ {\rho^i}} , \\
{F}^{i}_{.,0} &= {B}^{i}_{.,0}+\frac{{\Gamma_3}^{i-1}_{.,0}}{ \rho^i } ,  \\
{F}^{i}_{j,k} &= \text{sign}\left({B }^{ i}_{j,k} + \frac{{ \Gamma_{3} }^{i-1}_{j,k}}{{ \rho^i}}\right)\left( \left| {B }^{ i}_{j,k} + \frac{{ \Gamma_{3} }^{i}_{j,k}}{{ \rho^i}}  \right| - \frac{\lambda_3}{\rho^i} \right)_+ \ \ \text{ for } j \not= 0, k \not= 0.
\end{align*}

\item Update ${\Gamma}^i$ as follows: 
\begin{align*}
{\Gamma}_{1}^{i} &= {\Gamma_{1}}^{i-1}+ \rho^i\left( {B}^{i} - {D}^{i} \right), \\
{\Gamma}_{2}^{i} &= {\Gamma_{2}}^{i-1}+ \rho^i\left( {B}^{i} - {E}^{i} \right), \\
{\Gamma}_{3}^{i} &= {\Gamma_{3}}^{i-1}+ \rho^i\left( {B}^{i} - {F}^{i} \right).
\end{align*}

\end{enumerate}
\end{enumerate}

\subsubsection{Update for $B$ in Step 3(b)}
The update for $B$ in Step 3(b) is a least squares problem with a $n \times (p_1+1)(p_2+1)$ design matrix. Here we show that clever matrix algebra can be applied in order to avoid solving this least squares problem in each iteration.  For convenience, we omit the superscripts in Step 3(b).

Let $\wt{B}, \wt{D}, \wt{E}, \wt{F}, \wt{\Gamma}_1,  \wt{\Gamma}_2$, and $\wt{\Gamma}_3$ denote the vectorized versions of 
${B}, {D}, {E}, {F}, {\Gamma}_1,  {\Gamma}_2$, and ${\Gamma}_3$. And let $\wt{W}$ denote the $n \times (p_1+1)(p_2+1)$-dimensional  matrix version of $W$.
  Then the objective of Step 3(b) can be rewritten as
\begin{equation} 
%\begin{split}
\frac{1}{2} \left\| 
\left[ \begin{array}{c}
\frac{1}{\sqrt{n}}{y}\\
\frac{\rho {(\wt{D}+\wt{E}+\wt{F})}  - ({\wt{\Gamma}_{1}+\wt{\Gamma}_{2}+\wt{\Gamma}_{3}}) }{\sqrt{3\rho }} \\
\end{array} \right] - \left[  
\begin{array}{c}
\frac{1}{\sqrt{n}}\wt{W}\\
\sqrt{3\rho } {I}_{(1+p_1)(1+p_2)}\\
\end{array} \right]\wt{B} \right\|_F^2. %= \frac{1}{2} \|\bs{y} - \bs{W}\wt{B}\|_2^2. 
%\end{split}
\label{bup}
\end{equation}
%Minimizing \eqref{bup} with respect to $\wt{B}$ is very fast, provided that the inverse of $\frac{\wt{W}^T\wt{W}}{n}+3\rho^i {I}$ is available. 
Therefore, before performing the ADMM algorithm described in Section~\ref{sec:fullAlg}, we compute the SVD of $\wt{W}$. Then for each iteration of Step 3(b), the Woodbury matrix identity can be very quickly  applied  in order to minimize \eqref{bup}.

\subsection{\family\ for Generalized Linear Models}
\label{app:alg-logistic}
We now consider the extension of \family\ to GLMs (Section~\ref{sec:Extension}). The resulting ADMM algorithm is  as in Section~\ref{sec:fullAlg}, except that the update for $B$ in Step 3(b) now takes the form
\begin{equation}
\footnotesize
\underset{B\in \mathbb{R}^{(p_1+1)\times (p_2+1)} }{\text{argmin }}\ \  \frac{1}{n} l(WB) + \frac{3\rho^i}{2} \left\| \frac{1}{3\rho^i} \left[\rho^i (D^{i-1}+E^{i-1}+F^{i-1}) -(\Gamma_1^{i-1}+\Gamma_2^{i-1}+\Gamma_3^{i-1})\right] - B  \right\|_F^2.
\label{eq:glm}
\end{equation}
%To solve this problem, we perform a second-order Taylor expansion of \eqref{eq:glm}, in which we approximate the Hessian using the upper bound of $(1/4) I $. 
To solve this problem, we perform a second-order Taylor expansion of \eqref{eq:glm}, in which we approximate the Hessian using a multiple of the identity (e.g., for logistic regression, we use the upper bound of $(1/4)I$). 
  Details are omitted in the interest of brevity.

%For the extension to generalized linear models we note that the only difference to the \family\ algorithm of Section \ref{sec:fullAlg} is the update of $B$. The complete algorithm can be written in the following manner:

%\begin{enumerate}
%\item Initialize ${\rho^0,\ B^0,\ \Theta^0\text{ and }\Gamma^0}$.
%\item Initialize $\varepsilon^{pri}>0$, $\varepsilon^{dual}>0$.
%\item Repeat for $i = 1,2,3,... $ until $r^i < \varepsilon^{pri}$ and  $s^i <\varepsilon^{dual}$, where $r^i$ and $s^i$ are as defined in Section \ref{sec:fullAlg}.
%\begin{enumerate}
%\item Update $\rho^i$ as in step 3(a) in the algorithm of Section \ref{sec:fullAlg}
%\item  Update ${B^i}$ as the solution to the convex optimization problem \newline
%\begin{equation*}
%\begin{split}
%B^i &= \underset{B}{\text{argmin}} \  l(B) + \frac{3\rho^i}{2} \left\| \frac{1}{3\rho^i} \left[\rho^i(D^{i-1}+E^{i-1}+F^{i-1}) -(\Gamma_1^{i-1}+\Gamma_2^{i-1}+\Gamma_3^{i-1})\right] - B  \right\|_2^2.
%\end{split}
%\end{equation*}

%\item Update ${D^i},E^i,F^i,$ and $\Gamma^i$ as in steps 3(c)-(e) in the algorithm of Section \ref{sec:fullAlg}.
%\end{enumerate}
%\end{enumerate}
%For the case of logistic regression (\ref{eqn:objglm-reparametrized}), we can update $B$ via simple iterative algorithm which we outline in Appendix \ref{app:deriveLogistic}.

\section{Proofs of Results in Section \ref{sec:proposal}}
\label{app:proofs}

\begin{comment}
\begin{proof}[Proof of Lemma \ref{lemma:dual-prox}]
For convex $P(\cdot)$, we can find the minimum value of (\ref{main-eqn:gen-prox}) by standard sub-gradient methods, i.e. the solution $\hat{\beta}$ must satisfy 
\begin{align*}
0 &= \hat{\beta} - y + \lambda \left. \frac{\partial \|\beta\|}{\partial \beta} \right|_{\beta = \hat{\beta}},
\end{align*}
and if $\hat{\beta} = 0$ then we must have the equation
\begin{equation}
0 = \lambda \nabla_0  - y ,
\label{eqn:nabla0}
\end{equation}
where $\nabla_0$ is a sub-gradient of $\|\beta\|$ at $0$. By the definition of sub-gradients we have 
\begin{equation}
\|\beta\| \ge \nabla_0^T\beta.
\label{eqn:cond.}
\end{equation}
By the definition of the dual norm we have the tight inequality: 
\begin{align}
\nabla_0^T\beta\le \|\beta\| \|\nabla_0\|_{*}\ , 
\end{align}
where $\|\cdot\|_*$ is the dual norm of $\|\cdot\|$. The above two inequalities together allow us to specify the sub-derivative set for $\|\beta\|$ at $\beta  =  0$ given by: 
\begin{align}
\partial_0 &=  \left\{ \nabla_0: \|\nabla_0\|\le 1  \right\}. 
\end{align}
Finally, if $\nabla_0 \in \partial_0$ then $\lambda\| \nabla \|_{*} \le \lambda$ and thus for $\nabla_0$ to satisfy (\ref{main-eqn:nabla0})
we simply need $\|y\|_{*} \le \lambda$. 
\end{proof}
\end{comment}

\begin{proof}[Proof of Lemma \ref{lemma:hierNet-dual}]
The result follows from the definition of the dual norm. 
\begin{align*}
P_*(z) &= \sup\{ z^T\beta: P(\beta) \le 1 \}\\
&= \sup\{ z^T\beta: \max(|\beta_1|, \|\beta_{-1}\|_1) \le 1 \}\\
&= \sup\{ z^T\beta: |\beta_1|\le 1 \text{ and } \|\beta_{-1}\|_1 \le 1 \}\\
&= \sup\{ z_1\beta_1 + z_{-1}^T\beta_{-1}: |\beta_1|\le 1 \text{ and } \|\beta_{-1}\|_1 \le 1 \}\\
&= \sup\{ z_1\beta_1 : |\beta_1|\le 1 \} + \sup\{ z_{-1}^T\beta_{-1}: \|\beta_{-1}\|_1 \le 1 \}\\
&= |z_1|+\|z_{-1}\|_{\infty}.
\end{align*}

\end{proof}

\begin{proof}[Proof of Lemma \ref{lemma:hierNet-dualprob}]
Consider the series of equalities:
\begin{align*}
\min_{\beta} \ \frac{1}{2} \|y - \beta\|^2 + \lambda P(\beta) &=\min_{\beta} \max_{P_*(u) \le \lambda} \  \frac{1}{2} \|y - \beta\|^2 + \beta^Tu\\
&=   \max_{P_*(u)  \le \lambda}  \min_{\beta}  \frac{1}{2} \|y - \beta\|^2 + \beta^Tu\\
&=  \max_{P_*(u)  \le \lambda}  \frac{1}{2} \|y - (y-u)\|^2 + (y-u)^Tu\\
&=  \max_{P_*(u)  \le \lambda} y^Tu - \frac{1}{2} \|u\|^2 \\
&=  \max_{P_*(u)  \le \lambda} - \frac{1}{2} \|u-y\|^2+constant.
\end{align*}
This is equivalent to the problem
\begin{align*}
&\underset{u \in \mathbb{R}^p}{\text{minimize}}\  \frac{1}{2}\|y-u\|^2\\
&\text{subject to }\  |u_1| + \|u_{-1}\|_{\infty} \le \lambda,
\end{align*}
%By adding an extra constrained variable $\lambda_1$ it %follows that the above problem is equivalent to 
which, in turn, is equivalent to (\ref{eqn:dual}).
\end{proof}

\subsection{Proof of Theorem \ref{thm:hierNet}}

We consider the function 
\begin{equation}
f(\lambda_1) = \frac{1}{2} \|u(\lambda_1)-y\|^2,
\label{eqn:f-lam1}
\end{equation}
where $u(\lambda_1)$ is a vector-valued function of $\lambda_1$, as defined in (\ref{eqn:udef}). We wish to minimize this function over the interval $[0,\lambda]$. We will prove this theorem using a series of claims.
\begin{claim}
The function $f(\lambda_1)$ is convex on $\mathbb{R}$.
\end{claim}
\begin{proof}
Note that 
\begin{align}
(y_1 - u_1(\lambda_1))^2 &=  (y_1 - y_1)^2\bs{1}(|y_1|\le \lambda_1) + (y_1 - \lambda_1 \mbox{sign}(y_1))^2 \bs{1}(|y_1| >\lambda_1) \nonumber \\
&= (y_1 - \lambda_1 \mbox{sign}(y_1))^2 \bs{1}(|y_1| >\lambda_1) 
\label{clm:1a}
\end{align}
and 
%which is a positive, quadratic function of $\lambda_1 <|y_1|$ and zero otherwise and is thus a convex function. Similarly we have for $u_i$:
\begin{align}
(y_i - u_i(\lambda_1))^2 &= (y_i - (\lambda -\lambda_1)\mbox{sign}(y_i) )^2\bs{1}(\lambda_1 > \lambda - |y_i| ).
\label{clm:1b}
\end{align}
By inspection, both \eqref{clm:1a} and \eqref{clm:1b} are convex. The result follows from the fact that the sum of convex functions is convex. 
%which is positive, quadratic for $\lambda_1 > \lambda - |y_i|$ and 0  otherwise thus making it a convex function. Finally, we note that $f(\lambda_1)$ is a sum of functions of the form (\ref{clm:1a}) and (\ref{clm:1b}) and thus is convex. 
%
%
%\textcolor{red}{uh-oh, \eqref{clm:1a} and \eqref{clm:1b} don't look convex to DW}
\end{proof}

\begin{claim}
The derivative of $f(\lambda_1)$ is given by
\begin{align}
\frac{d}{d\lambda_1} f(\lambda_1) &= [\lambda_1 - |y_1|]\bs{1}(|y_1| > \lambda_1) +\sum_{i=1}^{p-1} \left[ \lambda_1 - z_{(i)} \right]\bs{1}(\lambda_1 > z_{(i)}),
\label{eqn:grad}
\end{align}
where $z$ is as defined in Theorem \ref{thm:hierNet}.
\end{claim}

\begin{proof}

%\textcolor{red}{DW has asked asad if we really need this whole proof... would like to simplify} 

Note that $f(\lambda_1)$ can be rewritten as 
$$f(\lambda_1) = (y_1 - \lambda_1 \mbox{sign}(y_1))^2 \bs{1}(|y_1| >\lambda_1) 
 + \sum_{i=2}^p  (y_i - (\lambda -\lambda_1)\mbox{sign}(y_i) )^2\bs{1}(\lambda_1 > \lambda - |y_i| ).$$
The result follows by inspection. 

\end{proof}

\begin{claim} \label{c3}
Define
\end{claim}
\begin{equation}
\lambda_1(m) = \frac{|y_1|+\sum_{j=1}^{m} z_{(j)} }{m+1}.
\label{eqn:lam1m}
\end{equation}
Then 
\begin{equation}
\underset{\lambda_1\in \mathbb{R} }{\text{argmin} } \ f(\lambda_1)  = \min_m \lambda_1(m).
\end{equation}
\begin{proof}
%First, note that we have assumed that $|y_1| + \| y_{-1} \|_\infty \geq \lambda$. Therefore, $|y_1| \geq z_{(1)}$.

%Now, define $\hat\lambda_1 \equiv \arg\min_{\lambda_1 \in \mathbb{R}} f(\lambda_1)$. Since \eqref{eqn:grad} evaluated at $\hat\lambda_1$ equals zero, and since we have already established that 
%$|y_1| \geq z_{(1)}$, 
%it follows that  $z_{(1)} < \hat{\lambda}_1 < |y_1|$.

 Let $z_{(p)} \equiv \infty$, and define $\lambda_1(m) \equiv \frac{|y_1| + \sum_{j=1}^{m} z_{(j)} }{m+1}$. 
The optimality conditions for $f(\lambda_1)$ guarantee that if $\lambda_1(m) \in (z_{(m)}, z_{(m+1)}]$, then $\hat{\lambda}_1 = \lambda_1(m)$. 
%$$\hat{\lambda}_1 = \lambda_1(m) \mbox{ \ \ if and only if \ \ } \lambda_1(m) \in (z_{(m)}, z_{(m+1)}].$$ 

If the set $\arg\min_m \lambda_1(m)$ contains a single element, then define $k \equiv \arg\min_m \lambda_1(m)$; otherwise, let $k$ be the smallest element of the set. 
%
%Define $\lambda_1(k) \equiv \min_m \lambda_1(m)$; is this minim
%
%, it holds 
 To complete the proof, it suffices to show that $\lambda_1(k) \in (z_{(k)}, z_{(k+1)}]$. 
 %This will complete the proof.
 
 First, we will show that $\lambda_1(k) > z_{(k)}$. By definition of $\lambda_1(k)$, we know that  $\lambda_1(k) < \lambda_1(k-1)$. In other words,
 $$ \frac{|y_1|+\sum_{j=1}^{k} z_{(j)} }{k+1} < \frac{|y_1|+\sum_{j=1}^{k-1} z_{(j)} }{k}.$$ 
 Rearranging terms, we find that 
$$ \left( {|y_1|+\sum_{j=1}^{k} z_{(j)} } \right) \left(1-\frac{1}{k+1}  \right) < {|y_1|+\sum_{j=1}^{k-1} z_{(j)} }.$$ 
Consequently,
$$z_{(k)} -  \frac{|y_1|+\sum_{j=1}^{k} z_{(j)} }{k+1} < 0.$$
%We have now established that $\lambda_1(k) \equiv \min_m \lambda_1(m)$, it holds that $\lambda_1(k) \in (z_{(k)}, z_{(k+1)}]$. This guarantees that $\lambda_1(k)$ minimizes $f(\lambda_1)$ over $\mathbb{R}$.
This means that $z_{(k)} < \lambda_1(k)$. % \textcolor{red}{UH-OH I wanted strict inequality?}

We now use a similar argument to show that $\lambda_1(k) \leq z_{(k+1)}$. By definition of $\lambda_1(k)$, we know that  $\lambda_1(k) \leq \lambda_1(k+1)$. In other words,
$$ \frac{|y_1|+\sum_{j=1}^{k} z_{(j)} }{k+1} \leq \frac{|y_1|+\sum_{j=1}^{k+1} z_{(j)} }{k+2}.$$ 
 Rearranging terms, we find that 
$$ \left( {|y_1|+\sum_{j=1}^{k} z_{(j)} } \right) \left(1+\frac{1}{k+1}  \right) \leq {|y_1|+\sum_{j=1}^{k+1} z_{(j)} } =   \left( |y_1|+\sum_{j=1}^{k} z_{(j)} \right) + z_{(k+1)}.$$ 
%Consequently,
%$$z_{(k)} -  \frac{|y_1|+\sum_{j=1}^{k} z_{(j)} }{k+1} \leq 0.$$
This implies that $\lambda_1(k) \leq z_{(k+1)}$. 
\end{proof}

Since  $f(\lambda_1)$ is convex, its minimizer in the interval $[0, \lambda]$ is simply the projection of its minimizer on $\mathbb{R}$ (given in Claim~\ref{c3}) into the interval. This completes the proof of Theorem~\ref{thm:hierNet}. 

\qed

\section{Degrees of Freedom for \family}
\label{app:dof}
\begin{proof}[Derivation of Claim \ref{claim:df}]

As mentioned in the main text,  an unbiased estimate for the degrees of freedom of (\ref{eqn:df-problem}) is given by 
\begin{equation}
\widehat{\df} = \sum_{i=1}^n \frac{\partial \hat{y}_i}{\partial y_i} = \mathrm{trace} \left( \frac{d \hat{y}}{dy} \right),
\label{dfest}
\end{equation}
provided that $\hat{y}(y)$ is almost differentiable. The proof that $\hat{y}(y)$ is almost differentiable follows from arguments similar to those in \citet{tibshirani2012degrees}.

We now derive an explicit form for \eqref{dfest}. 
To evaluate $\frac{d\hat{y}}{dy}$, we first note that
%Thus, for estimating the degrees of freedom of (\ref{eqn:df-problem}) we need to evaluate $\frac{d\hat{y}}{dy}$. We also need some regularity assumptions regarding the derivative $\frac{d\hat{\beta}_{\mathcal{A}}}{dy}$, but those are left out for brevity. 
%\noindent We evaluate $\frac{d\hat{y}}{dy}$ via the following: 
%\begin{enumerate}
%\item We first note that 
$\hat{\beta}_{\mathcal{A}}$, the solution of (\ref{eqn:df-problem}) restricted to the active set, takes the form
\begin{equation}
\hat\beta_{\mathcal{A}} = \underset{\beta_{\mathcal{A}}}{\text{argmin}} \  \left\{ \frac{1}{2} \|y - X_{\mathcal{A}}\beta_{\mathcal{A}} \|_2^2 + \sum_{d} \lambda_d P_d(A_d^{\mathcal{A}}\beta_{\mathcal{A}}) \right\}.
\label{eqn:full-gen}
\end{equation}
Therefore,  $\hat{\beta}_{\mathcal{A}}$ must satisfy 
\begin{equation}
-{X}_{\mathcal{A}}^T(y-{X}_{\mathcal{A}} \hat{\beta}_{ \mathcal{A}} )+ \sum_d \lambda_d (A^{\mathcal{A}}_d)^T  \dot{P}_d( A^{\mathcal{A}}_d \hat{\beta}_{ \mathcal{A}}) =0.
\end{equation}
We then differentiate with respect to $y$ and apply the chain rule, to obtain 
\begin{equation}
-{X}_{\mathcal{A}}^T + {X}_{\mathcal{A}}^T{X}_{\mathcal{A}} \frac{d\hat{\beta}_{\mathcal{A}}}{dy}+ \sum_d \lambda_d (A^{\mathcal{A}}_d)^T \ddot{P}_d (A^{\mathcal{A}}_d\hat{\beta}_{\mathcal{A}})  \left( A^{\mathcal{A}}_d \right) \frac{d\hat{\beta}_{\mathcal{A}} }{dy} = 0.
\end{equation}
 Solving for $\frac{d\hat{\beta}_{\mathcal{A}} }{dy}$ gives us
\begin{equation}
\frac{d\hat{\beta}_{\mathcal{A}} }{dy}  =  \left[ {X}_{\mathcal{A}}^T {X}_{\mathcal{A}}+ \sum_d \lambda_d (A^{\mathcal{A}}_d)^T \ddot{P}_d  (A^{\mathcal{A}}_d\hat{\beta}_{\mathcal{A}}) (A^{\mathcal{A}}_d) \right]^{-1} {X}_{\mathcal{A}}^T.
\end{equation}
Form the definition of $\hat{y} = X_{\mathcal{A}} \hat{\beta}_{ \mathcal{A}}$, we get
\begin{equation}
\frac{d\hat{y}}{dy}  = {X}_{\mathcal{A}} \frac{d\hat{\beta }_{\mathcal{A}} }{dy} =  {X}_{\mathcal{A}} \left[ {X}_{\mathcal{A}}^T {X}_{\mathcal{A}}+ \sum_d \lambda_d (A^{\mathcal{A}}_d)^T \ddot{P}_d (A^{\mathcal{A}}_d\hat{\beta}_{\mathcal{A}}) (A^{\mathcal{A}}_d) \right]^{-1} {X}_{\mathcal{A}}^T.
\end{equation}
In order to make this derivation  entirely rigorous, we would need to show that $\hat\beta$ is unique, and that with probability one, within some neighbourhood of $y$, the active set $\mathcal{A}$ does not change as a function of $y$. 
%\end{enumerate}

\end{proof}

\end{document}